\documentclass[acmlarge]{acmart}
\makeatletter
\newcommand{\myconfshort}{\acmConference@shortname}
\newcommand{\myconffull}{\acmConference@name}
\newcommand{\myconfdate}{\acmConference@date}
\newcommand{\myconfloc}{\acmConference@venue}
\AtBeginDocument{
  \fancypagestyle{firstpagestyle}{
    \fancyhead{}%
    \fancyfoot[C]{}%
  }
  \fancyhf{}
  \fancyhead[LO]{\@headfootfont\shorttitle}%
  \fancyhead[RE]{\@headfootfont\@shortauthors}%
  \fancyhead[LE]{\@headfootfont\footnotesize \myconfshort, \myconfdate, \myconfloc}%
  \fancyhead[RO]{\@headfootfont\footnotesize \myconfshort, \myconfdate, \myconfloc}%
  \fancyfoot[C]{}%
}
\makeatother
\acmBooktitle{\conffull\@ (\confshort), \confdate, \confloc}

\AtBeginDocument{%
  }

\copyrightyear{2026}
\acmYear{2026}
\setcopyright{cc}
\setcctype{by}
\acmConference[FAccT '26]{The 2026 ACM Conference on Fairness, Accountability, and Transparency}{June 25--28, 2026}{Montreal, QC, Canada}
\acmBooktitle{The 2026 ACM Conference on Fairness, Accountability, and Transparency (FAccT '26), June 25--28, 2026, Montreal, QC, Canada}
\acmDOI{10.1145/3805689.3812302}
\acmISBN{979-8-4007-2596-8/2026/06}

\begin{document}

\title[Characterizing the Pareto Frontier of Algorithmic Decision Systems]{Fairness vs Performance: Characterizing the Pareto Frontier of Algorithmic Decision Systems}

\author{Mieke Wilms}
\authornote{Both authors contributed equally to this research.}
\affiliation{%
  \institution{Zurich University of Applied Sciences / University of Zurich}
  \city{Zurich}
  \country{Switzerland}
}
\email{mieke.wilms@zhaw.ch}

\author{Christoph Heitz}
\authornotemark[1]
\affiliation{%
  \institution{Zurich University of Applied Sciences}
  \city{Zurich}
  \country{Switzerland}}
\email{christoph.heitz@zhaw.ch}

\renewcommand{\shortauthors}{Wilms and Heitz}

\begin{abstract}
Designing fair algorithmic decision systems requires balancing model performance with fairness toward affected individuals: More fairness might require sacrificing some performance and vice versa, yet the space of possible trade-offs is still poorly understood. We investigate fairness in binary prediction-based decision problems by conceptualizing decision making as a multi-objective optimization problem that simultaneously considers decision-maker utility and group fairness. We investigate the set of Pareto-optimal decision rules for arbitrary utility functions for decision maker, arbitrary population distributions, and a wide range of group fairness metrics implementing different justice-theoretic principles such as egalitarianism, prioritarianism, and Rawlsian approaches. We optimize over all possible decision rules based on a given feature vector $\vec{x}$, including rules that explicitly depend on the protected attribute. 

Under these circumstances, we find that the Pareto frontier consists of deterministic, group-specific threshold rules applied to individuals' success probability $p(\vec{x})$. This complements existing optimality theorems from literature which, for specific fairness constraints, posit lower-bound threshold rules only. However we also show that, depending on the used fairness metric, the Pareto frontier may include \emph{upper-bound} threshold rules, thus preferring individuals with lower success probabilities.  We show that the location of the Pareto frontier depends only on population characteristics, utility functions and fairness score, but not on the technical design of the algorithm –- our findings hold for pre-, in-, and post-processing approaches alike. We present an example where an in-processing method adopts group-specific threshold rules without having access to the protected attribute.  

Our results generalize existing optimality theorems for fairness-constrained classification and extend them to generalized fairness metrics and fairness principles, and to partial fairness regimes. This paper connects formal fairness research with legal and ethical requirements to search for less discriminatory alternatives, offering a principled foundation for evaluating and comparing algorithmic decision systems.

\end{abstract}




\maketitle

\section{Introduction} \label{sec: Introduction}

Algorithmic decision systems are widely used in society: ranging from the financial industry \cite{fuster_predictably_2017,kozodoi_fairness_2022,moscato_2021_benchmark} to the healthcare sector \cite{beam_2018_big,habehh_2021_machine,shailaja_2018_machine} to governmental institutions \cite{van_2021_digital,lagioia_2023_algorithmic}. Such systems are designed to achieve the goals of a decision maker (DM) while making decisions about individuals in a population, in the following referred to as decision subjects (DS), thereby affecting their lives.  
Concerns about discrimination, unfairness and social justice have been studied in numerous publications during the last years. This paper adopts a \textit{group fairness} approach\footnote{For other notions of fairness, the reader may consult \cite{verma_2018_fairness}.}, which focuses on systematic disadvantages arising from individuals' membership of socially salient groups defined by sensitive attributes such as gender, ethnicity, or age. 

The research literature of the last years has developed many different approaches to design and implement fairness-aware decision making system, particularly in prediction-based decision systems. This are decision systems which involve the prediction of a variable (typically denoted by the symbol $Y$) which informs the decision but is unknown at the time of decision making. Common approaches are often categorized into pre-, in- and post-processing methods \cite{caton_fairness_2020}.

In contrast to focusing on prediction models as such, we study prediction-based decision systems holistically by conceptualizing them as \textit{decision rules}, that is, mappings $f:\vec{x}\mapsto D$ from observable feature vectors $\vec{x}$ to decisions $D$. Each such system produces at least two outcomes of interest: the extent to which it achieves the DM's goals, and the degree of (un)fairness it generates with respect to the relevant social groups. Developers of fairness-aware decision systems need to take both outcomes into account and find solutions that, ideally, maximize both fairness and goal achievement at the same time. However, it is well known that these objectives are often in tension: maximizing goal achievement often occurs at the price of fairness, while enforcing fairness may reduce performance, a phenomenon referred to as the accuracy-fairness tradeoff \cite{pessach_2022_review}. 

This tension naturally leads to questions such as: What is the maximum achievable level of goal achievement for a given minimum level of fairness, and vice versa. This question also arises in the legal context of discrimination \cite{wachter_2021_fairness,weerts_2023_algorithmic}: one of the questions to be answered in a discrimination lawsuit is whether the decision maker could have achieved their goal with less inequality, since there is a legal obligation to search for \textit{Less Discriminatory Alternatives} \cite{black_LDA_2024,laufer_LDA_2025}.

This paper addresses these questions by following a line of work that addresses fairness as a multi-objective optimization problem \cite{liu_2022_accuracy,valdivia_2021_fair,zhang_2022_mitigating,wei_2022_fairness}. Within a two-dimensional outcome space defined by goal achievement and fairness, we study the Pareto frontier, that is, the set of decision rules that are Pareto-optimal with respect to these objectives. The central research question of this paper is: Is it possible to characterize the Pareto frontier of a prediction-based decision system in the solution space of \emph{all possible technical implementations} of decision systems? What can be said about Pareto-optimal decision rules in general, independent on how they are implemented?

Our analysis focuses on binary decision problems with decisions $D\in\{0,1\}$ with a binary decision-critical variable $Y\in\{0,1\}$. We model the outcome of "goal achievement" with a DM utility function $U$ which is assumed to be maximized. Following \cite{baumann_2023_distributivejusticefoundationalpremise}, fairness is captured by a generalized fairness score $FS$, derived from comparisons of expected DS utilities between groups, which is minimized or maximized depending on the underlying principle of justice. This general approach includes classical confusion-matrix-based fairness metrics (Positive rate, True/False Positive Rate, Positive or Negative Predictive Value, etc.  \cite{barocas_fairness_2023,verma_2018_fairness}) as special cases, but also allows explicitly modeling the effect of decisions (harm or benefit) and using alternative principles of distributive justice that go beyond equality.

\subsubsection*{Contributions}
The main contribution of this paper is a characterization of the Pareto frontier of binary prediction-based decision systems $\vec{x}\mapsto D$ in the joint utility-fairness space, across all technical designs and for a broad class of fairness metrics\footnote{The Python-code used for the results presented in this paper is available on \url{https://github.com/mcwilms/Pareto-Frontier-of-Binary-Decision-Systems}}. For determining the Pareto frontier, we consider all possible decision rules, including those that depend on the protected attribute $a$\footnote{Group-dependent decision rules are always possible to implement if $a\in\vec x$, but this is not a necessary condition: even without access to the sensitive attribute $a$ it may be possible to implement group-specific decision rules, see the discussion in Section \ref{subsec: the optimal decision rule} and the example in Section \ref{sec: comparative study}.}. To the best of our knowledge, no prior work has provided such a characterization under comparably general conditions. 

Specifically, we make the following contributions:

\begin{itemize}
    \item We show that, for all fairness metrics based on the evaluation of expected utilities of decision subjects across groups, the Pareto frontier over the space of all possible decision systems allowing for group-dependent decision rules  consists of group-dependent deterministic threshold rules applied to the probabilities $P[Y=1|\vec x]$. These rules may be lower-bound threshold, upper-bound thresholds or combinations thereof. While classical papers such as \cite{hardt_equality_2016,corbett-davies_algorithmic_2017} have proven for some fairness metrics that enforcing fairness might lead to lower-bound threshold rules, we show that, for implementing Pareto-optimal decision rules, it might be necessary to adopt upper-bound threshold rules which prefer individuals with a lower $P[Y=1]$ over higher $P[Y=1]$, a phenomenon which has been described and discussed as \textit{cherry-picking} \cite{fleisher_fair_2021,favier_cherry_2025} or \emph{within group unfairness} \cite{baumann_enforcing_2022}. 
    
    \item We show that the occurrence of upper-bound threshold rules as Pareto-optimal solutions is related to using generalized DS utility functions which model the impact of decisions, as opposed to focusing on classical confusion matrix based metrics. Our results embed existing results on fairness-constraint optimization \cite{hardt_equality_2016,corbett-davies_algorithmic_2017,baumann_enforcing_2022} as special cases, but extend them to a broader range of fairness metrics. These include classical equality-based fairness metrics as well as more general notions based on generalized DS utility functions or other patterns of distributive justice, such as maximin, prioritarianism or sufficientarianism.
    
    \item The derived Pareto frontier establishes a technology-agnostic benchmark independent on the technical design and implementation of the decision system, thus including pre-processing, in-processing and post-processing methods. We show that the Pareto frontier depends only on population characteristics, the DM and DS utility functions, and the chosen fairness score $FS$ derived from the expected DS utility across groups.
    
    \item We show that for determining the location of the Pareto frontier it is enough to have (an estimate of) the probability distributions $g(p|a)$, instead of a Bayes optimal estimate for $p(\vec x)$.  
   
\end{itemize}

Note that all these results are independent on how the decision system is designed. This means that even very different approaches will tend to reproduce the same decision rules when trying (and by trying) to optimally combine decision maker utility and fairness.

\section{Related research}
\subsection{Fairness metrics} \label{sec: fairness scores}
Measuring fairness is a prerequisite to designing fair decision systems, and researchers have introduced many different fairness metrics to measure algorithmic fairness \cite{dwork_fairness_2012,barocas_fairness_2023,verma_2018_fairness,pessach_2022_review,mitchell_algorithmic_2021}. Different approaches such as group fairness, individual fairness, counterfactual fairness, or causal fairness have been proposed and developed In this paper, we focus on group fairness, where groups are defined by a sensitive attribute $a$ (e.g. gender or ethnicity). 

Table \ref{tab: group fairness metrics} contains some examples of popular metrics used to assess group fairness.  
By comparing popular fairness metrics across different groups, a \emph{Fairness Score} $FS$ can be defined that measures the degree of (un)fairness as a continuous variable. For example, $FS=|P[D = 1 \mid a = 0]-P[D = 1 \mid a = 1]|$ is a fairness score for selection rate. \emph{Fairness criteria} can be constructed by requiring a specific value of a fairness score, e.g. $FS=0$. It is important to note that most fairness criteria can not be satisfied at the same time \cite{friedler_2018_impossibility}, requiring a choice and making algorithmic fairness not only a mathematical problem, but also a philosophical one \cite{hertweck_justice-based_2023,weerts_2022_goodhartslaw}.

\begin{table}[t]
\centering
\begin{tabular}{l l}
\hline
\textbf{Metric name} & \textbf{Mathematical notation} 
\\
\hline
Selection Rate 
  & $P[D = 1 \mid A = a]$ 
  \\

TPR (True Positive Rate)
  & $P[D = 1 \mid Y = 1, A = a]$ 
  \\
FPR (False Positive Rate
  & $P[D = 1 \mid Y = 0, A = a] $ 
  \\

PPV (Positive Predictive Value) 
  & $P[Y = 1 \mid D = 1, A = a]$ 
  \\
FOR (False Omission Rate) 
  & $P[Y = 1 \mid D = 0, A = a]$ 
  \\
\hline
\end{tabular}
\caption{Examples of metrics used for algorithmic fairness assessment. $D=\{0,1\}$ denotes the binary decision, $Y=\{0,1\}$ denotes the decision critical variable and $A$ denotes the set of groups. 
}
    \label{tab: group fairness metrics}
\end{table}

In this paper, we focus on binary decision systems or binary classifiers. In this context, fairness metrics are often derived from the elements of the confusion matrix, e.g. via True/False Positive Rate (TPR/FPR), True/False Negative Rate (TNR/FNR), Positive/Negative Predictive Value (PPV/NPV), etc. \cite{barocas_fairness_2023}, focusing on inequalities of decisions. However, unfairness defined as disadvantage \cite{barocas_fairness_2023} is rather an inequality of the \emph{impact} of decisions than of the decisions themselves, which led to more recent works focusing on the impact on decision subjects, which might depend not only on $D$, but on a combination of $D$ and $Y$ (see, e.g. \cite{hu_2018_welfare,baumann_2023_distributivejusticefoundationalpremise}). For example, a decision for a medical treatment ($D=1$) based on a prediction of having a specific disease ($Y=1$) may have a positive impact on a patient actually having this disease ($Y=1$), but a negative impact on anyone else, due to severe side effects. The classical confusion-based metrics, in contrast, implicitly assume that a positive decision has the same benefit/harm for individuals with both $Y=0$ and $Y=1$. 

Baumann et al. (2023) \cite{baumann_2023_distributivejusticefoundationalpremise} have suggested to model impact with a \emph{decision subject (DS) utility function} $V$ with matrix components $(v_{ij})_{i,j\in\{0,1\}}$, where $v_{ij}$ is the benefit of decision $D=i$ for an individual with outcome $Y=j$:
\begin{equation*}
    (v_{ij})_{i,j\in\{0,1\}} = \ \ 
\begin{array}{cc}
   & \begin{array}{cc} \scriptstyle Y=0 & \scriptstyle Y=1 \end{array} \\ 
   \begin{array}{r} \scriptstyle D=0 \\ \scriptstyle D=1 \end{array} & 
   \left( \begin{array}{cc}
      v_{00} & v_{01} \\
      v_{10} & v_{11}
   \end{array} \right)
\end{array}
\end{equation*}
They show that the classical confusion-based metrics (such as selection rate, TPR, PPV, etc.) are expectation values of $V$ over the four possible outcomes $Y\in\{0,1\} \times D\in\{0,1\}$, either unconditional expectations $E[V|a]$ or conditioned expectations $E[V|Y=j]$ or $E[V|D=j]$, with specific choices of $V$ and $j$. For example, selection rate $P[D=1|a]$ equals $E[V|a]$ for $(v_{ij})=(0,0;1,1)$. FPR equals $E[V|Y=0,a]$ with the same $(v_{ij})$, and PPV corresponds to $E[V|D=1,a]$ with $(v_{ij})=(0,1;0,1)$ (see Appendix \ref{app: utility fairness} for a derivation). 

This leads to a generalized version of equality-based fairness as (possibly conditioned) equality of such expectation values across groups:
\begin{equation*}
    E[V|J=j,a]=E[V|J=j,a'], \text{ for different groups }a, a'  
\end{equation*}
with $J\in\{\emptyset,Y,D\}$ ($J=\emptyset$ standing for the unconditional case). The most important equality-based fairness criteria (e.g. Demographic Parity, FNR parity, PPV parity) are special cases, but the generalized form establishes a much broader range of equality-based fairness metrics. Partial fairness can be measured by a \emph{Fairness score} $FS$ which, for example, can be chosen as the absolute difference $|E[V|J=j,a]-E[V|J=j,a']|$.  

Equality-based fairness criteria as defined above are the most common forms in the fairness literature. Philosophically, they fall in the category \textit{egalitarianism}, a theory of distributive justice that sees some form of equality in distributions of burden and profit as morally desirable \cite{arneson_egalitarianism_2002,hertweck_distributive_2024}. Beyond egalitarian fairness, Baumann et al. \cite{baumann_2023_distributivejusticefoundationalpremise} extend their approach also to other justice principles like \textit{Rawls max-min} \cite{rawls_justice_1999} (where fairness is assessed by the utility of the worst-off group); \textit{prioritarianism} \cite{holtug_prioritarianism_2007} (where the utility of disadvantaged group weights heavier); or \textit{sufficientarianism} \cite{frankfurt_equality_1987} (where fairness is reached if all groups meat a minimum level of utility). They show that, also in these cases, fairness scores can be constructed as functions of the expected DS utilities for the different groups $E[V|J=j,a], a\in A$ (see also \cite{baumann_2023_distributivejusticefoundationalpremise,hertweck_distributive_2024}).

Our approach builds on this generalized approach of measuring fairness by average impact on individuals of different groups and deriving a scalar fairness score $FS$. Depending on the chosen principle of distributive justice the objective is to minimize $FS$ (e.g. when egalitarianism is pursued via the absolute difference in decision subject utility between groups) or to maximize $FS$ (e.g. in Rawls max-min principle).

\subsection{Enforcing fairness via Constraint Optimization} \label{sec: implementing fairness}
To implement fairness in settings where a DM seeks to maximize their own benefit, many works incorporate a constrained optimization approach, maximizing DM utility under a fairness constraint. Several papers have analyzed and solved this problem for specific cases. For statistical parity, true positive rate (TPR) parity and false positive rate parity (FPR), it has been proven that the optimal decision rule satisfying these constraints is given by a group-dependent lower-bound threshold rule applied to Bayes-optimal predictions \cite{corbett-davies_algorithmic_2017,hardt_equality_2016}. For positive predictive value (PPV) parity, false omission rate (FOR) parity and sufficiency (PPV parity and FOR parity simultaneously) it has been proven that the optimal decision rules is a group-dependent threshold rule which might include upper-bound thresholds \cite{baumann_enforcing_2022}. 

These works \cite{hardt_equality_2016,corbett-davies_algorithmic_2017,baumann_enforcing_2022} are limited in three ways: First, they are based on confusion-matrix based metrics (FPR, PPV, etc.) and do not allow for generalized metrics based on expected DS utility functions. Second, they seek to implement parity with as underlying moral principle  egalitarianism. Other relevant notions of justice as Rawls max-min principle or prioritarianism (see Section \ref{sec: fairness scores}) are not covered. Third, they provide optimality theorems only for fully enforcing parity (fairness as constraint), not for partial fairness.

In this paper we provide an optimality theorem which overcomes these limitations while including the classical optimality theorems, by allowing (a) for generalized fairness metrics based on arbitrary DS utility matrices $Y$, (b) for non-egalitarian fairness notions, and (c) for arbitrary levels of fairness requirements, ranging from requiring full fairness to ignoring fairness at all.

\subsection{Fairness as a multi-objective optimization problem} \label{sec: moo in ADS}

A more recent line of ML fairness research treats fairness as a multi-objective optimization problem, focusing on trade-offs between a DM goal and fairness, often under a perspective of fairness-accuracy tradeoff \cite{pessach_2022_review,tang_2023_theoretical}. This perspective includes solutions that satisfy fairness only partially and thus expands the solution space compared to constraint-optimization methods of \cite{corbett-davies_algorithmic_2017,hardt_equality_2016,baumann_enforcing_2022}. Different papers provide approaches to find the Pareto frontier for binary classification problems (which is equivalent to our perspective of binary decision systems), formulating the task as a ML problem and applying in-processing methodologies, e.g. \cite{liu_2022_accuracy,valdivia_2021_fair,zhang_2022_mitigating,wei_2022_fairness}. While using different algorithmic approaches, each of these papers is based on a specific class of ML models (e.g. neural networks \cite{wei_2022_fairness,zhang_2022_mitigating}, stochastic multi-gradient algorithm \cite{liu_2022_accuracy}, or logistic regression \cite{valdivia_2021_fair}), showing improvement over other approaches, but not providing any optimality properties in the full space of \emph{all conceivable classification algorithms} as our paper does. Furthermore, a characterization of the solutions on the Pareto frontier is not possible which is a standard limitation of in-processing approaches. 

\cite{silvia_2020_general,jang_2022_group} follow a different approach, by analyzing the classification task as a thresholding problem on the (estimated) probability $P[Y=1]$ of the individuals, leading to Pareto-optimal solutions which are better interpretable. However, both papers limit the space of decision rules to the class of lower-bound threshold rules, as opposed to allowing arbitrary decision rules. Our paper shows that  this is a serious limitation and may lead to suboptimal frontiers.

Note that the aforementioned methods construct a Pareto frontier only over a restricted subset of decision rules, rather than over the full space of possible decision rules, and therefore do not resolve the question whether a less discriminatory alternative (LDA) exists \cite{black_LDA_2024}. \cite{cen_2025_auditsresourcedataaccess} propose a scaling-law based method to test for the existence of an LDA using test data, a small subsample of training data and an estimate of the both the contested model's size and the amount of data used to train the contested model. Their approach enables comparison of the contested model against a multiplicity of alternative models to assess the existence of an LDA. In contrast, we propose a methodology to derive an outer bound of the Pareto frontier that requires only (a sufficiently accurate estimate of) the probability distributions of $P[Y=1|\vec x,a]$ for individuals of protected group $a\in A$. 

We are not aware of any papers which (a) consider the complete space of possible decision rules (or classification models), (b) allow for generalized utility functions of both DM and DS, and (c) characterize the Pareto-optimal solutions by their decision rule. 
Conceptually, our paper fundamentally differs from the current state-of-the-art research in algorithmic fairness in that it is based on a decision-theoretical perspective (\textit{How to characterize Pareto-optimal decisions?}) rather than a learning-focused ML perspective (\textit{How to derive the Pareto frontier from learning data?}). This change in perspective allows for novel insights which, in turn, may inspire new ML approaches.

\section{Determining the Pareto frontier}
In this section, we derive the Pareto frontier of a binary prediction-based decision system, capturing the optimal utility-fairness trade-off for a broad class of fairness notions. We first introduce the problem statement and notation (Section \ref{subsec: problem statement}), then define the utility functions (Section \ref{subsec: utility-based}). Next, we characterize the Pareto-optimal decision rules (Section \ref{subsec: the optimal decision rule}) and finally we show how to construct the Pareto frontier (Section \ref{subsec: construction of PF}). 

\subsection{Problem statement and notation} \label{subsec: problem statement}
We consider a prediction-based decision system where a decision maker (DM) takes a binary decision $D$ on individuals (decision subjects [DS]) of a population P, based on a feature vector $\vec x$ which includes the sensitive (or protected) attribute $a\in A$ that indicates the group membership\footnote{For covering the space of all possible decision rules, we allow decision rules which depend on the sensitive attribute $a$. In practice, this might be legally prohibited. We will discuss situations where $a$ is not used for decision making in Sections \ref{sec: comparative study} and \ref{sec:Conclusions}.}. Furthermore, each individual is characterized by a binary random outcome variable $Y$ whose value is not known at the time of decision making, yet is critical to the decision making. For example, a bank has to decide whether to give a loan to an individual ($D=1$) or not ($D=0$), not knowing whether the individual will repay the loan at the end of maturity ($Y=1$) or not ($Y=0$). In the following, we assume that, if $Y$ was known, the optimum decision rule for the DM would be $D=1$ if $Y=1$ and $D=0$ if $Y=0$. 

The population P is characterized by a joint probability distribution over $(Y,\vec x)$. For binary $Y$, this joint probability distribution is fully determined by the function $p(\vec x):=P[Y=1|\vec x]$ (as $P[Y=0|\vec x]=1-P[Y=1|\vec x]$) and the probability density $h(\vec x)$ over the feature space\footnote{For any given $\vec x$, the outcome $Y$ is a random variable. The uncertainty of $Y$ can include both the intrinsic uncertainty of $Y$ for each individual as well as uncertainty due to the aggregation of different individuals with the same feature vector. $P[Y=1|\vec x]$ is the probability of $Y=1$ of all individuals with the same feature vector $\vec x$.}. Prediction-based decision making exploits dependencies between $\vec x$ and $Y$ to optimize the decision. 

We denote the utility of the DM, created by taking the decision $D$, by a random variable $U$ which depends on both $D$ and $Y$. For each individual decision, the decision-process has four possible outcomes, depending on the combination $(D,Y)$. Each of these outcomes corresponds to a different value $u_{ij}$ of the DM utility, corresponding to $D=i,Y=j$. For example, $u_{10}$ is the utility for an individual with $Y=0$ who is given a decision $D=1$ (utility of a false-positive). We capture this in a matrix $(u_{ij})_{i.j\in\{0,1\}}$. For an ideal outcome $D=Y$ (in case of known $Y$), the utility matrix needs to fulfill the following consistency conditions: $u_{11}> u_{01}$ and $u_{00}> u_{10}$. 

Each of the four outcomes also corresponds to a DS utility $V$, quantifying the amount of benefit (or harm) that the individual receives from the decision - which also may depend on both $D$ and $Y$ \cite{hertweck_group_2024}. We capture this in the utility matrix $(v_{ij})_{i.j\in\{0,1\}}$ of the DS. In contrast to $(u_{ij})$, there are no consistency conditions on the values of $(v_{ij})$ and we allow arbitrary $V$. We assume that both utility matrices are constant across individuals of the same group, but $(v_{ij})$ might depend on the group.

\subsection{Decision policies and utilities}\label{subsec: utility-based}
Any data-based decision system is based on a decision rule, which maps $\vec x$ to a decision $d(\vec x)$: $\vec x \mapsto d(\vec x)$. We include randomized decision policies, i.e. $d(\vec x) = P[D=1|\vec x]$ with $0\leq d(\vec x) \leq 1$ \footnote{E.g., $d(\vec{x})=0.4$ means that there is a 40 percent chance that an individual with feature vector $\vec{x}$ is assigned a decision $D=1$.}. A deterministic policy is a special case where $d(\vec x) \in \{0,1\}$ for all $\vec x$. 

For an individual with feature vector $\vec x$, the expected DM utility is given by 
\begin{equation} \label{eq: expected Udm}
    \begin{split}
        E[U|\vec{x}] & = p(\vec{x}) d(\vec{x}) u_{11} + (1-p(\vec{x})) d(\vec{x}) u_{10} + (1-p(\vec{x}))(1-d(\vec{x})) u_{00}+p(\vec{x})(1-d(\vec{x})) u_{01} \\
        & =d(\vec{x})(\alpha p(\vec{x})+\beta)+\gamma p(\vec{x})+u_{00}
    \end{split}
\end{equation}
with $\alpha=u_{11}-u_{10}+u_{00}-u_{01}$, $\beta=u_{10}-u_{00}$ and $\gamma=u_{01}-u_{00}$. From $u_{11}> u_{01}$ and $u_{00}> u_{10}$, it follows that $\alpha>0$, $\beta<0$, and $\alpha+\beta>0$. 

Similarly, the expected DS utility is given by $E[V|\vec{x}] = d(\vec{x})(\tilde{\alpha} p(\vec{x})+\tilde{\beta})+\tilde{\gamma} p(\vec{x})+v_{00}$, with $\tilde{\alpha}=v_{11}-v_{10}+v_{00}-v_{01}$, $\tilde{\beta}=v_{10}-v_{00}$ and $\tilde{\gamma}=v_{01}-v_{00}$, with no constraints on $\tilde\alpha$ and $\tilde\beta$. 

The expected DM utility for a randomly selected individual of the population is given by $E[U]=\int_X E[U|\vec{x}]h(\vec{x})d\vec{x}$, integrated over the feature space $X$. 
However, note that $E[U|\vec x$] only depends on $p$ and $d$. Instead of integrating over $\vec x$, we can equally well integrate over $p$:
\begin{equation} \label{eq: int Udm}
    E[U] = \int_0^1 (d(p)(\alpha p+\beta)+\gamma p+ u_{00}) \cdot g(p) \: dp
\end{equation}
where $g(p)$ is the probability distribution of $p$ for the given population, and $d(p)$ is the fraction of individuals with a given $p$ who receive a positive decision: $d(p)=P[D=1|p]$ 
\footnote{Formally, $g(p)$ is given by $g(p)=\int \delta(p(x)-p) h(\vec x) \: d\vec x$, with $\delta$ the Dirac delta function. The function $d(p)$ is given by $d(p)=\frac{1}{g(p)}\int d(\vec x) \cdot \delta(p(\vec x)-p) \hspace{1mm} h(\vec x) \: d\vec x $, denoting the average decision for all individuals with given $p$.}.
The function $g(p)$ is a characteristic of the population, while $d(p)$ characterizes the decision rule, applied to the given population. The same formalism can be applied to the expected DS utility $E[V]$, using the parameters $(\tilde\alpha, \tilde\beta, \tilde\gamma,v_{00})$.

We can calculate $E[U]$ as a mixture of group-specific utilities: $E[U]=\sum_i E[U|a=i] P[a=i]$, with $P[a=i]$ the fraction of individuals of $P$ belonging to group $i$. The expected utilities $E[U|a]$ and $E[V|a]$ of group $a$ can be calculated analogously to Equation (\ref{eq: int Udm}), replacing $d(p)$ by $d(p|a)$ and $g(p)$ by $g(p|a)$. Here, $d(p|a)$ characterizes the decision rule applied to group $a$, and $g(p|a)$ the distribution of $p$ for group $a$. 

Note that $d(p)$ (or, equivalently, the set of $d(p|a)$ for all groups $a$) fully characterizes the decision rule for the sake of calculating the average utilities both for DM and DS. Two different decision rules $d(\vec x)$ which lead to the same $d(p|a)$ for all groups can be considered equivalent, as they lead to the same values of $E[U|a]$ and $E[V|a]$. Optimizing a decision rule can thus be done by optimizing $d(p)$. 

In the following, for notational convenience for the upcoming proofs, we approximate Equation (\ref{eq: int Udm}) using the Riemann integral. We divide the probability space $[0,1]$ into $N$ bins of size $1/N$, where all individuals with a probability $p$ that falls in bin $i$ (i.e. $p\in[\frac{i-1}{N},\frac{i}{N}]$) are assigned probability $p_i=\frac{i-0.5}{N}$ and decision rule $d_i\in[0,1]$: 
\begin{equation}\label{eq: riemann sum E(U)}
      E[U] = \sum_{i=1}^N (d_i(\alpha p_i+\beta)+\gamma p_i+u_{00}) w_i
\end{equation}
where $w_i=\int_{\text{bin }i}g(p)dp$ denotes the weight of each summand. In this discretization, a decision rule $d(p)$ is represented by a vector $(d_1,d_2, \dots,d_N)$. Analogously, for $E[V]$ we derive
\begin{equation}\label{eq: riemann sum E(V)}
      E[V] = \sum_{i=1}^N (d_i(\tilde{\alpha} p_i+\tilde{\beta})+\tilde{\gamma} p_i+v_{00}) w_i
\end{equation}
Note that we introduce this discretization only for easier notation of the proofs following in section \ref{subsec: the optimal decision rule}. We will always assume the limit $N\rightarrow\infty$, corresponding to the correct expectation values. 

The expected DM and DS utility for a member of group $a$ can be approximated analogously by $E[U|a]=\sum_{i=1}^N(d_i(\alpha p_i+\beta)+\gamma p_i+u_{00})w_i^a$ and $E[V|a]=\sum_{i=1}^N(d_i(\tilde{\alpha} p_i+\tilde{\beta})+\tilde{\gamma} p_i+v_{00})w_i^a$, with $w_i^a$ the weight of the summand given the distribution $g(p|a)$ for group $a$. For the sake of notational convenience, we will write $w_i$ instead of $w_i^a$ if the context makes clear which is meant.

\subsection{Pareto-optimal decision rules}\label{subsec: the optimal decision rule}
Let's assume DM and DS utility matrices $(u_{ij})_{i,j\in\{0,1\}}$ and $(v_{ij})_{i,j\in\{0,1\}}$. 
Let $\mathcal{D}$ be the class of all possible decision rules $d(p)$. For a group $a\in A$, any given decision rule $d\in \mathcal{D}$ yields $\big(E[U|a],E[V|J=j,a]\big )$. From the conditions on the values $u_{ij}$ stated in section \ref{subsec: utility-based} it is easy to see that the decision rule optimizing DM utility without fairness constraint is given by a threshold rule 
\begin{equation}\label{eq: theoretical max threshold rule}
    d_i^*=
    \begin{cases}
        1 & \text{for }p_i>-\beta/\alpha \\
        0 & \text{otherwise}
    \end{cases}
\end{equation}

However, we are interested in decision rules that jointly optimize $E[U]$ and the defined fairness score $FS$. Before characterizing the decision rules that optimize $\big(E[U],FS\big)$, we first focus on one group by characterizing decision rules $d(p|a)$ that maximize $E[U|a]$ for all possible values of $E[V|J=j,a]$.

\begin{theorem}\label{thm: threshold decision rule}
    Given a DM utility matrix $(u_{ij})$ and a DS utility matrix $(v_{ij})^a$. The expected DS utility $E[V^a|J=j,a]$, $a\in A$, may be conditioned on  $J\in\{\emptyset, D,Y\}$. Let $\mathcal{D_\eta}$ be the set of all possible decision rules $d(p|a)$ with given $E_{d}[V|J=j,a]=\eta$.

    Then it holds: For any feasible value of $\eta$, the decision rule $d^*\in\mathcal{D_\eta}$ that maximizes $E_{d}[U|a]$ takes one of the following forms 
    \begin{equation} 
    d_{lb}=\begin{cases}
        1 & \text{for } p\geq t_a\\
        0 & \text{otherwise}
    \end{cases} \quad \quad\quad or \quad\quad
    d_{ub}=\begin{cases}
        1 & \text{for } p< t_a\\
        0 & \text{otherwise} 
    \end{cases}
    \end{equation}
    where $t_a\in[0,1]$ denotes a lower-bound (lb) or upper-bound (ub) threshold.
\end{theorem}

The full proof of Theorem \ref{thm: threshold decision rule} can be found in Appendix \ref{app: proof theorem threshold rule}. Here, we provide the basic structure of the proof for the unconditional case $J=\emptyset$. We consider the discretized version of decision rules $d=(d_1,d_2,...,d_N)$ (see Section \ref{subsec: utility-based} ). A lower-bound threshold rule is characterized by $d_i=0$ for $i<i_t$, $d_i=1$ for $i>i_t$ and $d_{i_t}\in[0,1]$ for some bin index $i_t$ that contains the threshold. Analogously, an upper-bound threshold rule is characterized by $d_i=0$ for $i>i_t$, $d_i=1$ for $i<i_t$ and $d_{i_t}\in[0,1]$. 

Let us consider an arbitrary decision rule $\hat d(p|a)$ with $E_{\hat d}[V|a]=\eta$, leading to $E_{\hat d}[U|a]=U_0$. Assume that $\hat d(p|a)$ is not a lower- or upper-bound threshold rule (as characterized before). Then we can construct a new decision rule $d^*$ which keeps $E[V|a]$ unchanged, but increases $E[U|a]$, by changing two elements ($d_j,d_k$) of the decision vector simultaneously. We always assume $j<k$ in the following. 

A simultaneous small change $(\Delta_j,\Delta_k$) of ($d_j,d_k$), i.e. $d_j\mapsto d_j+\Delta_j$ and $d_k\mapsto d_k+\Delta_k$, leads to a change $\Delta E[U|a]$ of $E[U|a]$, and a change $\Delta E[V|a]$ of $E[V|a]$.
From Equation (\ref{eq: riemann sum E(V)}) it follows that $\Delta E[V|a] = \Delta_j (\tilde\alpha p_j +\tilde\beta) w_j +\Delta_k (\tilde\alpha p_k +\tilde\beta) w_k$. Thus, we have $\Delta E[V|a]=0$ (and thereby $E[V|a]$ unchanged) if
\begin{equation} \label{eq:delta dj main}
        \Delta_j=-\frac{\tilde{f}(p_k)w_k}{\tilde{f}(p_j)w_j}\Delta_k
\end{equation}
where $\tilde{f}(p)=\tilde \alpha p+\tilde \beta $ and $w_j$ and $w_k$ denote the weight of the bins. Imposing this condition on the simultaneous change ($\Delta_j,\Delta_k$) results in a change $\Delta E[U|a]$ given by 
\begin{equation}\label{eq: delta E(U) main}
    \Delta E[U|a]=\Delta_k w_k\tilde f (p_k)\big(\mathcal F (p_k)-\mathcal F (p_j)\big)
\end{equation}
with $\mathcal F(p)=f(p)/\tilde{f}(p)$, where $f(p)=\alpha p+\beta$. The two possible forms of $\mathcal F(p)$ are shown in Figures \ref{fig: plots f functions}(e)-(f) (Appendix \ref{app: plots f-functions}). An increase of $E[U|a)]$ ( i.e. $\Delta E[U|a]>0$) requires either a positive or a negative $\Delta k$, depending on the signs of $\tilde f (p_k)$ and $(\mathcal F (p_k)-\mathcal F (p_j))$. As $\tilde f(p)$ can take various shapes with respect to $f(p)$ (as shown in Figures \ref{fig: plots f functions}(a)-(d)) there are several cases to consider.

A simple case is that the function $\tilde{f}(p)$ is strictly positive in $[0,1]$. Then, $\frac{\tilde f(p_k)w_k}{\tilde f(p_j)w_j}$ is positive, and $\tilde f(p_k)\big(\mathcal F (p_k)-\mathcal F (p_j)\big)$ is positive (see $\mathcal F(p)$ in Figure \ref{fig: plots f functions}(e) on the right side of the vertical asymptote). Thus, a positive $\Delta E[U|a]$ requires $\Delta_k>0$ (Eq. (\ref{eq: delta E(U) main})) and $\Delta_j<0$ (Eq. (\ref{eq:delta dj main})). 
This, in turn, requires that $d_j>0$ and $d_k<1$. It is easy to see that there exists at least one pair ($j,k$) with $d_j>0$ and $d_k<1$, unless the following condition holds: if $d_j>0$, then, for all $k>j$, $d_k=1$. This, however, characterizes a lower-bound threshold rule. Thus, unless $d$ represents a lower-bound threshold rule, one can construct a new rule with the same $E[V|a]$ and a higher $E[U|a]$. 

The full proof in Appendix \ref{app: proof theorem threshold rule} leads through all possible cases. It shows that it is always possible to find a change $(\Delta_j,\Delta_k)$ fulfilling Equation (\ref{eq:delta dj main}) which leads to a positive $\Delta E[U|a]$, unless $d$ is a lower- or upper-bound threshold rule, which proves theorem \ref{thm: threshold decision rule} for any discretized approximation of calculating $E[U]$. Taking the limit $N\to\infty$ proves theorem \ref{thm: threshold decision rule} for the exact calculation of $E[U]$. The proofs for the conditional cases $E[V|J=j,a]$ with $J=Y$ and $J=D$ are given in Appendix \ref{app: addendum proof 3.1}.

Theorem \ref{thm: threshold decision rule} has a restricted view on a single group $a$. In the following, we consider multiple groups and derive conditions for Pareto-optimal decision rules with respect to expected DM utility and fairness score. 

First, we introduce the concept of \textit{Pareto optimality}.

\begin{definition}\label{def: pareto optimality}
    Let $\mathcal{D}$ be the set of decision rules $d(p)$, where every $d\in\mathcal{D}$ yields $\big(E_{d}[U],FS_{d}\big)$. A decision rule $d\in\mathcal{D}$ is \textit{Pareto optimal} with respect to $E_{d}[U]$ and $FS$ if and only if it is not Pareto dominated, that is, there does not exist any $\hat d\in \mathcal{D}$ such that $E_{\hat d}[U]\geq E_{d}[U]$ and $FS_{\hat d}\geq FS_d$\footnote{In case the objective is to maximize $FS$, the condition is $FS_{\hat d}\leq FS_d$.} with at least one strict inequality.
\end{definition}

 The set of all Pareto-optimal decision rules is defined as the \textit{Pareto frontier}. The following optimality theorem specifies the Pareto-optimal solutions in the two-dimensional space $(E[U],FS)$.

\begin{theorem} \label{thm: Pareto frontier}
    Given a DM utility matrix $(u_{ij})$, (possibly group-specific) DS utility matrices $(v_{ij})^a$, and a fairness score $FS$ which is a function of the (possibly conditioned) expected DS utilities $E[V^a|J=j,a], a\in A$, with $J\in\{\emptyset,D,Y\}$, and $FS$ is either to be maximized or to be minimized. We consider the set of all possible decision rules $\mathcal D$ including those that allow independent decision rules for different groups $a\in A$. 
    
    Then it holds that any decision rule $d \in \mathcal D$ which is Pareto-optimal with respect to $(E[U],FS)$ implements group-specific deterministic threshold rules on $p(\vec{x})$ for all groups $a\in A$. 
\end{theorem}

\begin{proof}
Assume that there exists a decision rule $\hat{d}$ on the Pareto frontier that does not correspond to a deterministic threshold rule for at least one group in $A$, which we will call $\bar{a}$. 
$\hat{d}$ applied to group $\bar a$ yields $E_{\hat{d}}[U|\bar{a}]$ and $E_{\hat{d}}[V^a|J=j,\bar a]$ . We now construct a decision rule $d^*$ that is the same as $\hat{d}$ except for group $\bar{a}$. For group $\bar{a}$, according to Theorem \ref{thm: threshold decision rule}, there must exist a deterministic threshold rule with $E_{d^*}[U|\bar{a}]>E_{\hat{d}}[U|\bar{a}]$, while $E_{d^*}[V^{\bar a}|J=j,\bar{a}]=E_{\hat{d}}[V^{\bar a}|J=j,\bar{a}]$. Changing the decision rule for group $\bar{a}$ to  $d^*$ thus leads to an increased expected DM utility as $E[U]=\sum_{a\in A} E[U|a]\cdot P[A=a]$, while $FS$ does not change. Thus, $\hat{d}$ is not Pareto-optimal according to Definition \ref{def: pareto optimality}, which leads to a contradiction.  
\end{proof}

Theorem \ref{thm: Pareto frontier} implies that, for a given population of individuals with feature vectors $\vec{x}$, there exists an upper bound on achievable DM utility for each level of fairness, as well as an upper bound on achievable fairness for each level of DM utility. No decision rule $\vec{x}\mapsto d(\vec{x})$ applied to the same population can exceed this theoretical boundary. Consequently, the resulting Pareto frontier defines the best utility-fairness tradeoff across all technical approaches used to construct the decision system (e.g. pre-, in-, or post-processing), and thus manifests a technology-agnostic benchmark. 

Theorem \ref{thm: Pareto frontier} is widely applicable as it holds for arbitrary populations and group compositions, for any (reasonable) DM utility function and any (possibly group-specific) DS utility function, and for arbitrary scalar fairness scores calculated from $E[V|J=j,a]$ where $J\in\{D,Y,\emptyset\}$. 
It generalizes earlier results \cite{corbett-davies_algorithmic_2017,baumann_enforcing_2022,hardt_equality_2016} to fairness, freely specified DS utility functions that capture how decisions generate benefits or harms for individuals, and a wide range of fairness notions, including non-egalitarian ones. 

The theorem states that, for any given group, the Pareto-optimal decision rule may be a lower-bound threshold rule (as in the seminal papers of 2016-17 \cite{corbett-davies_algorithmic_2017, hardt_equality_2016}) but, under certain conditions, may equally well be an upper-bound threshold rule, i.e. $D=1$ for $p<t_a$, which is somewhat counter-intuitive, but has already been observed and discussed in recent works such as \cite{baumann_enforcing_2022,favier_cherry_2025}. 

Note that Theorem \ref{thm: Pareto frontier} is based on the assumption that the decision rules for different groups can be optimized independently. This seems to imply that the decision system needs access to the sensitive attribute $a$. However, this is not necessarily the case. While including $a$ in the feature vector $\vec x$ guarantees that group-specific decision rules can be implemented, and thus is a \emph{sufficient} condition of fulfilling the prerequisites of Theorem \ref{thm: Pareto frontier}, this is not a \emph{necessary} condition. Consider a case where $\vec x$ does not contain $a$, but other features which are sufficiently correlated with $a$. Then, $a$ can be reconstructed from $\vec x$ (at least approximately), which allows implementing group-specific decision rules. An example is given in Section \ref{sec: comparative study} where a trained decision algorithm does not use $a$ but still leads to group-specific thresholds as Pareto-optimal decision rules. We may conclude that any algorithm optimized for determining Pareto-optimal solutions can be expected to approximate the optimal solution given by Theorem \ref{thm: Pareto frontier}, as far as its model structure allows.

\subsection{Construction of the Pareto frontier}\label{subsec: construction of PF}

Theorem \ref{thm: Pareto frontier} characterizes the Pareto-optimal decision rules, but does not fully specify the actual Pareto-optimal rules: while all Pareto-optimal solutions are group-dependent threshold rules, not all group-dependent threshold rules are Pareto-optimal. However, it is easy to see how the theorem leads to a simple constructive approach to determine the Pareto frontier. In a first step, all combinations of group-specific ub (upper-bound) and lb (lower-bound) threshold rules (with thresholds $t_a \in [0,1]$) are evaluated with respect to $E[U|a]$ and $E[V|J=j,a]$, for all groups $a\in A$, from which $E[U]$ and $FS$ can be determined. In a second step, the Pareto frontier is determined by identifying all non-Pareto-dominated combinations. 

Furthermore, since $E[U]$ and $E[V|J=j,a]$ are continuous functions of the thresholds, the Pareto frontier can be approximated by limiting the evaluation of \emph{all} threshold rules (which are infinitely many) to finite grid of thresholds in $[0,1]$. For the case of two groups, a grid of $M$ threshold values in $[0,1]$ for each group leads to  $2M\times2M$ decision rules, the non-dominated ones being Pareto-optimal solutions. The larger $M$ is, the more accurately the shape of the Pareto frontier is sampled.

The results in Section \ref{subsec: the optimal decision rule} are still theoretical, as they are based on the assumption that $p(\vec x)$ is known. However, our framework suggest at least two different approaches of estimating the location of the Pareto frontier from empirical data. 

The \textbf{first approach} is straight-forward: It consists in using machine learning methods for predicting $P[Y=1|\vec x]$, resulting in an estimate $\hat p(\vec x)$ for all $\vec x$. For a test set of data samples $(\vec x, Y, a)$, one can test different threshold combinations with both ub and lb threshold rules. This leads to estimates of $E[U]$ and $FS$. The resulting empirical Pareto frontier is an estimate of the true Pareto frontier. Stochastic variations are expected due to both the imperfection of the prediction model and the finiteness of the test set. 

Note, however, that the location of the Pareto frontier can be determined even without using a predictor of $P[Y=1|\vec x]$ over the feature space. From Equations (\ref{eq: riemann sum E(U)}) and (\ref{eq: riemann sum E(V)}) we see that, for determining $E[U|a]$ and $E[V|J=j,a]$, we only need to know the distributions $g(p|a)$. Thus, a \textbf{second approach} is to estimate the distributions $g(p|a)$, and to use Equations (\ref{eq: riemann sum E(U)}) and (\ref{eq: riemann sum E(V)}) for evaluating different threshold rules (see Section \ref{sec: illustrative example} for an application of this approach). While $g(p|a)$ still needs to be determined from training data, it is arguably much easier to estimate this one-dimensional densities than to find a good predictor of $p(\vec x)$ in a multi-dimensional feature space \footnote{It is beyond the scope of this paper to discuss appropriate estimation methods. Note, however, that an important characteristic of $g(p|a)$ is the empirical base rate.}.

\section{Illustrative example on synthetic data} \label{sec: illustrative example}
We consider a synthetic population consisting of two groups, labeled 0 and 1, with different probability distributions $g(p|a)$ that are assumed to be known (see appendix \ref{app: synthetic dataset sec 4} for details). We choose $FS=|E[V|a=0]-E[V|a=1]|$ which is to be minimized. We construct the Pareto frontier as described in Section \ref{subsec: construction of PF} (second approach) with a threshold step size of $0.001$ (which means we tested $10^3$ threshold values per group, resulting in $4\times10^6$ decision rules). For the considered example of two groups, four types of decision rules are possible: \textit{lb-lb} (both groups use lower-bound thresholds), \textit{lb-ub} (group 0 uses lower-bound thresholds and for group 1 upper-bound), \textit{ub-lb} (vice versa) and \textit{ub-ub} (both groups use upper-bound thresholds).

We set the DM utility matrix to $(u_{ij})=(0,0;-0.5,1)$, meaning that a correct decision $D=1$ yields a DM utility of 1, an incorrect decision $D=1$ yields a loss of 0.5, and decision $D=0$ has no effect on the DM utility. Figure \ref{fig: synth data examples} illustrates, for two different DS utility matrices $(v_{ij})$, a set of decision rules that yield an expected DM utility $E[U]$ and fairness score $FS=|E[V|a=0]-E[V|a=1]|$. In all three subfigures, point A denotes the decision rule that maximizes $E[U]$. Recall that the DM-optimal decision rule without a fairness constraint depends only on $(u_{ij})$ and is given by the group-independent lower-bound threshold rule 
\begin{equation} \label{eq: example max dm U rule}
    d=
    \begin{cases}
        1 & \text{for }p\geq-\beta/\alpha=0.33 \\
        0 & \text{otherwise}
    \end{cases}
    \end{equation}

The Pareto frontier consists of the decision rules that maximize $E[U]$ for given values of $FS$, which is minimized in this example. Graphically, this corresponds to the lower-right outer boundary of the point clouds. We indicate the Pareto frontier separately for each of the four rule types (lb-lb, lb-ub, ub-lb and ub-ub) to visualize which types of decision rules constitute the outer bound.

The first DS utility matrix $(v_{ij})=(0,0;1,1)$ used corresponds to the fairness metric selection rate (see Appendix \ref{app: utility fairness}), which, in the case of $FS=0$, is equal to the egalitarian fairness criterion of statistical parity (also known as demographic parity). For this specific DS utility matrix we find $\tilde{\alpha}=v_{11}-v_{10}+v_{00}-v_{01}=0$ and $\tilde{\beta}=v_{10}-v_{00}=1$, thus $\tilde f(p)=1$. Using Theorem \ref{thm: Pareto frontier} (and insights of the proof in Appendix \ref{app: proof theorem threshold rule}) we find that Pareto-optimal decision rules are always lower-bound threshold rules for the selected fairness metric. This not only confirms the result of Hardt et al. \cite{hardt_equality_2016} and Corbett-Davies et al. \cite{corbett-davies_algorithmic_2017} (that holds for the egalitarian case where $FS=0$), but extends them by going beyond the egalitarian principle and constructing the whole Pareto frontier. Figure \ref{fig: synth data examples}(a) confirms this by showing that the whole Pareto frontier consists only of lower-bound threshold rules. Point B denotes the decision rule that satisfies $FS=0$ (i.e. fulfilling statistical parity).

The second DS utility matrix we test is $(v_{ij})=(0,0;-1,1)$\footnote{Note that the chosen DS utility matrix is in fact a realistic option, e.g. for medical decision making of a treatment that can be lifesaving to patients who are actually sick ($Y=1$), but causes severe side effects to patients do not suffer from the disease ($Y=0$) where $D=1$ denotes that the patient will get the treatment.}. For this specific matrix we find $\tilde\alpha=2$ and $\tilde\beta=-1$, resulting in $\tilde f(p)=2p-1$ which intersects the $x$-axis at $\tilde p^*=0.5$. From the proof of Theorem \ref{thm: threshold decision rule} (Appendix \ref{app: proof theorem threshold rule}) we derive that for $\tilde p^*\in[0,1]$, the DM utility maximizing threshold rule may be either a lower-bound or upper-bound threshold rule. Figures \ref{fig: synth data examples}(b)-(c) show that, while the upper part of the Pareto frontier ($FS>0.18$) is given by a lb-lb threshold rule (red line), most of the Pareto frontier ($FS<0.18$) consists of a lb-ub threshold rule (blue line). This means that individuals of group 1 with lower success probabilities are preferred over those with higher probabilities. Although this outcome may appear counterintuitive, given classical optimality theorems, it is consistent with previous findings \cite{baumann_enforcing_2022,favier_cherry_2025}. 

Point B in Figure \ref{fig: synth data examples}(b) denotes the decision rule satisfying $FS=0$. We see that the lb-ub rule achieving $FS=0$ attains significantly higher expected DM utility than the corresponding lb-lb rule.

In Figure \ref{fig: synth data examples}(c) we zoom in to the region where DM utility is maximal, showing a lb-lb solution. This is expected as the maximum DM utility (point A) is always achieved by a lb-lb threshold rule (see Equation (\ref{eq: example max dm U rule})). At point C, the lb-lb and the lb-ub Pareto frontiers intersect and yield identical performance. The corresponding decision rules are equivalent: 
\begin{equation} 
d_{\text{lb-lb}}=
\begin{cases}
    1 & \text{for }p\geq0.42 \text{ and } A=0 \\
    1 & \text{for }p\geq0 \text{ and } A=1 \\
    0 & \text{otherwise}
\end{cases}
\quad\quad\quad
d_{\text{lb-ub}}=
\begin{cases}
    1 & \text{for }p\geq0.42 \text{ and } A=0 \\
    1 & \text{for }p<1 \text{ and } A=1 \\
    0 & \text{otherwise}
\end{cases}
\end{equation}
 
\begin{figure}[t]
    \includegraphics[width=0.33\linewidth]{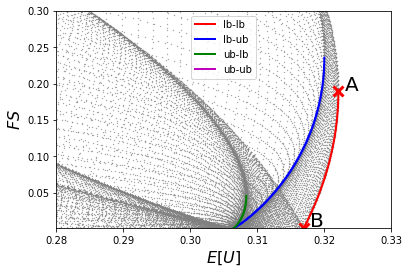}
    \includegraphics[width=0.33\linewidth]{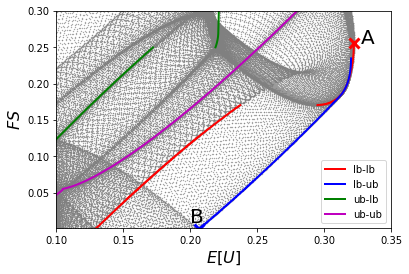}
    \includegraphics[width=0.33\linewidth]{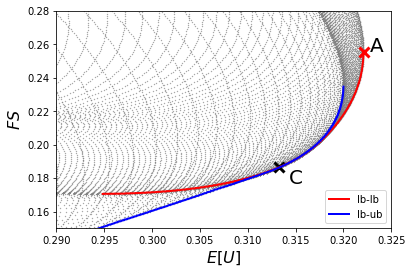}
    \\
    \small (a) $(v_{ij})=\begin{array}{cc}
   & \begin{array}{cc} \scriptstyle Y=0 & \scriptstyle Y=1 \end{array} \\ 
   \begin{array}{r} \scriptstyle D=0 \\ \scriptstyle D=1 \end{array} & 
   \left( \begin{array}{cc}
      0 & 0 \\
      1 & 1
   \end{array} \right)
\end{array}$ \quad\quad\quad
    (b) $(v_{ij})=\begin{array}{cc}
   & \begin{array}{cc} \scriptstyle Y=0 & \scriptstyle Y=1 \end{array} \\ 
   \begin{array}{r} \scriptstyle D=0 \\ \scriptstyle D=1 \end{array} & 
   \left( \begin{array}{cc}
      0 & 0 \\
      -1 & 1
   \end{array} \right)
\end{array}$\quad\quad\quad
    (c) $(v_{ij})=\begin{array}{cc}
   & \begin{array}{cc} \scriptstyle Y=0 & \scriptstyle Y=1 \end{array} \\ 
   \begin{array}{r} \scriptstyle D=0 \\ \scriptstyle D=1 \end{array} & 
   \left( \begin{array}{cc}
      0 & 0 \\
      -1 & 1
   \end{array} \right)
\end{array}$
    \caption{The Pareto frontier for different settings of DS utility matrices. The plots show a sample of all possible threshold combinations (grey dots), the Pareto frontier for only lower-bound threshold rules for both groups (red), the Pareto frontier for lower-bound for group 0 and upper-bound for group 1 (blue), the Pareto frontier for upper-bound for group 0 and upper-bound for group 1 (green) and the Pareto frontier for only upper-bound threshold rules for both groups (purple). }
    \label{fig: synth data examples}
\end{figure}

\section{Comparative study: in-processing fairness via a stochastic multi-objective approach}\label{sec: comparative study}
In this section we compare our method for finding the Pareto frontier of a binary decision system with an in-processing approach proposed by Liu and Vicente \cite{liu_2022_accuracy}, the PF-SMG (\textit{Pareto front stochastic multi-gradient}) algorithm. PF-SMG implements fairness in decision systems by formulating a stochastic multi-objective optimization problem that aims in jointly minimizing logistic regression loss as a measure of accuracy (capturing model performance), and disparate impact as measure of fairness, using an approximation of the so-called CV score $CV=|P[D=1|a=0]-P[D=1|a=1]|$ \cite{calders_2010_three}. The algorithm results in a Pareto frontier with respect to performance and the CV approximation. Plotted in the (Accuracy,CV)-space, these solutions are somewhat scattered, but still indicating the shape of the Pareto frontier. 

The metrics used in \cite{liu_2022_accuracy} directly map to our utility framework by choosing $(v_{ij})=(0,0;1,1)$ (thus $E[V|a]=P[D=1|a]$, see Appendix \ref{app: utility fairness}) and $(u_{ij})=(1,0;0,1)$ (thus $E(U)=P[D=Y]$).

We use the Adult Income dataset\footnote{The Adult Income dataset is from the US Census bureau and the task is to predict whether a given adult makes more that $\$50000$ a year. The dataset is available in the UCI Machine Learning Repository \url{https://archive.ics.uci.edu/dataset/2/adult}.} (also used in \cite{liu_2022_accuracy}). Liu and Vicente \cite{liu_2022_accuracy} assume that the decision algorithm has no access to the sensitive attribute $a$, and accordingly trained their classifiers (using $a$ only for evaluation purposes). Analogously, we trained a logistic regression model on the features $\vec x$ excluding $a$ (using $5000$ samples), resulting in individual success probability estimates $\hat p_{LR}(\vec x)$. Applying $100\times100$ threshold combinations with lower-bound threshold rules\footnote{For the chosen DS utility matrix $(v_{ij})=(0,0;1,1)$, the Pareto-optimal solutions need to be lower-bound threshold rules (compare Fig. \ref{fig: synth data examples}(a) in Section \ref{sec: illustrative example}).} and evaluating both $E(U)$ and $|E[V|a=0]-E[V|a=1]|$ for each combination, on a test set with $40222$ samples, we construct our Pareto frontier as the set of all non-dominated combinations. The PF-SMG algorithm was run using code (including the parameter settings) provided by Liu and Vicente \cite{liu_2022_accuracy}\footnote{\url{https://github.com/sul217/MOO_Fairness}}, with the same train-test split as in our experiment.

Figure \ref{fig: section 5}(a) shows the two Pareto frontiers achieved by Liu and Vicente\footnote{Compare to Fig. 1(d) in \cite{liu_2022_accuracy}.} and our approach. Our Pareto frontier is clearly dominating the one found by Liu and Vicente. At a first glance, this might seem not surprising since our approach uses the sensitive attribute $a$ for the decision making (by implementing group-dependent decision rules), while \cite{liu_2022_accuracy} only allow decision rules not using $a$.

\begin{figure}[t]
    \includegraphics[width=0.33\linewidth]{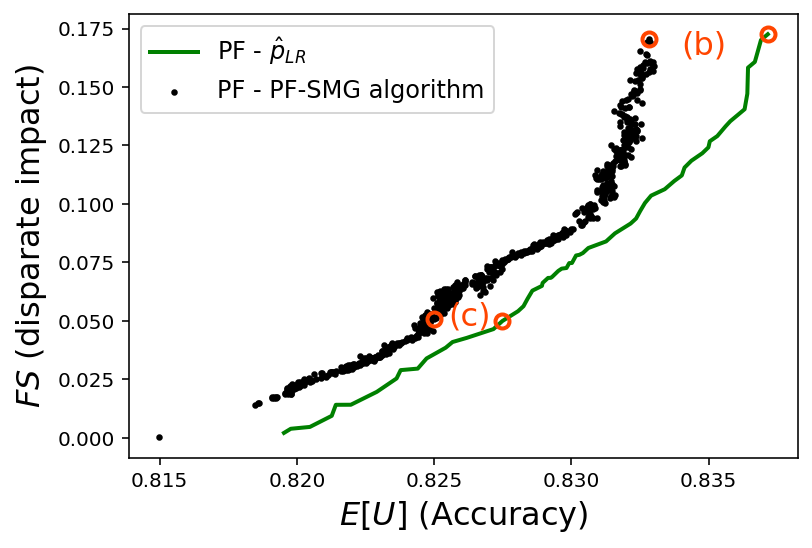}
    \includegraphics[width=0.33\linewidth]{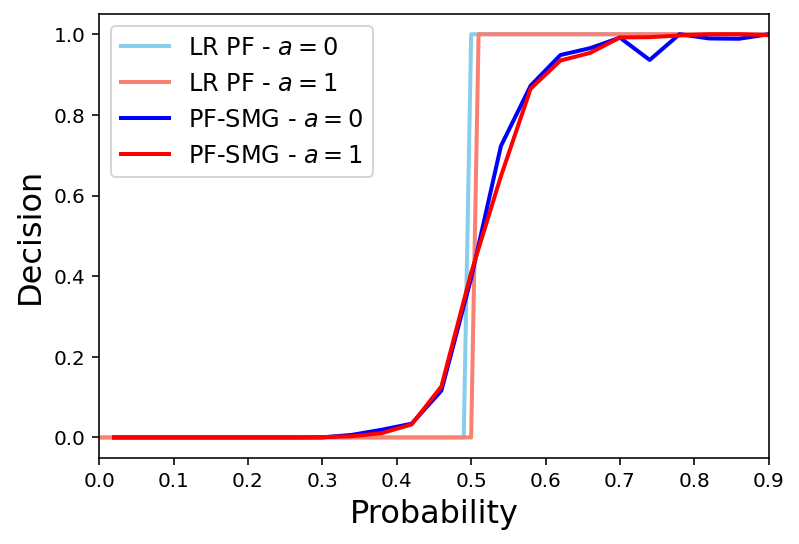}
    \includegraphics[width=0.33\linewidth]{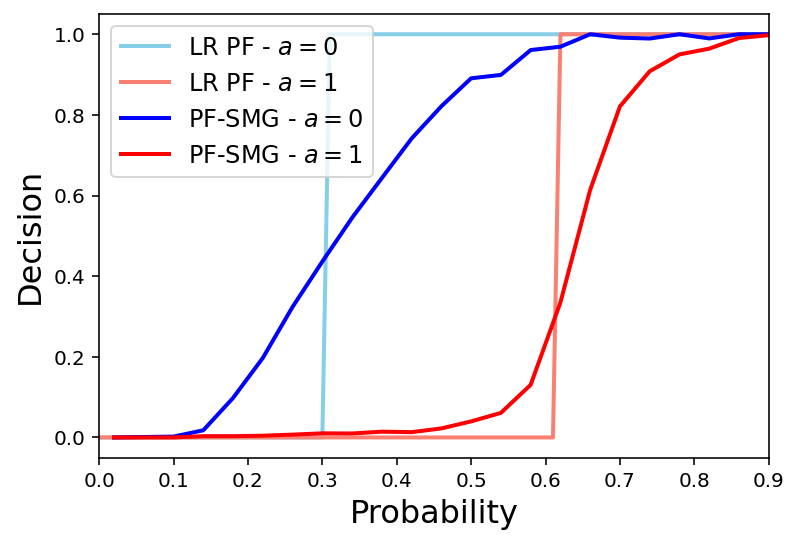}
    \\
    \small (a) \quad\quad\quad\quad\quad\quad\quad\quad\quad\quad\quad\quad\quad\quad\quad\quad
    (b) \quad\quad\quad\quad\quad\quad\quad\quad\quad\quad\quad\quad\quad\quad\quad\quad
    (c)
    \caption{Comparison of the Pareto frontier (PF) obtained by the PF-SMG algorithm and our approach (performed on a Logistic Regression estimate $\hat p_{LR}$) for the Adult Income datatset. Figure (a) shows both Pareto frontiers with Disparate Impact (DI) as fairness score and Accuracy as DM utility. Figures (b)-(c) show for two levels of DI the corresponding decision rules per group. }
    \label{fig: section 5}
\end{figure}

However, when we compare the Pareto-optimal decision rules obtained from both approaches, it becomes clear that this not depending on $a$ does not explain the performance difference. For this, we reconstruct the decision vector $d(p)$ of both our solution and the PF-SMG solution by partitioning the interval $p\in[0,1]$ in 25 bins, assigning each individual to the associated bin according to their estimated probability $\hat p_{LR}$, and calculating the fraction of individuals with $D=1$ in each bin, which yields an estimate of $d(p)=P[D=1|p]$.

We evaluated the decision rules at two points on the Pareto frontier  (indicated by (b) and (c) in Fig. \ref{fig: section 5}(a)). Location (b) represents the accuracy maximizing decision rule for both approaches, while location (c) represent a decision rule with less disparate impact.  
Figure \ref{fig: section 5}(b) shows the accuracy maximizing decision rule for both approaches. Recall that the DM utility maximizing decision rule (without fairness constraint) is given by Equation (\ref{eq: theoretical max threshold rule}) where $p_i=-\beta/\alpha=0.5$. This holds for both groups. Our approach reproduces this theoretical result (with slight stochastic deviations due to the finiteness of the training set). The PF-SMG approach implements a smoothed-out approximation of this threshold rule, as expected independent on the group. However, since the threshold rule is not exactly reproduced, the accuracy is lower than in our solution (see Fig. \ref{fig: section 5}(a), points (b)). Since, for this situation, the optimum decision rule is the same for both groups, our approach has no advantage over PF-SMG. Thus the performance difference between points (b) in Fig. \ref{fig: section 5}(a) cannot be attributed to explicitly using $a$, but is rather an effect of structural or algorithmic limitations of PF-SMG. 

Figure \ref{fig: section 5}(c) shows the decision rules corresponding to $FS=0.05$ (points (c) in Fig. \ref{fig: section 5}(a)). We find that the PF-SMG decision rule mimics the optimal decision rule by actually implementing group-specific threshold rules -- notably without having access to the sensitive attribute $a$. This shows that excluding the sensitive attribute $a$ in the decision process does not necessarily exclude implementing group-specific threshold rules (compare the discussion in Section \ref{subsec: construction of PF}).

Furthermore, a reduction of disparate impact requires a decrease of the threshold for group $a=0$, and an increase for group $a=1$ (compare  Fig. \ref{fig: section 5}(b) with Fig. \ref{fig: section 5}(c)). This change in the decision rule is realized by the PF-SMG solution, however again with a smoothed-out threshold. Obviously, even if PF-SMG may not be able to implement the optimum decision rule due to structural limitations and missing access to $a$, its search for Pareto-optimal solutions leads to decision rules which approximate the theoretically optimal group-dependent threshold rules. Even if this is an anecdotic observation only, we expect that this generalizes to all approaches for finding Pareto-optimal solutions. We leave the validation of this expectation to further research.

The presented example also shows how any given decision system can be compared with the Pareto frontier derived with our approach: Comparing the achieved values $(E(U),FS)$ with the set of Pareto-optimal solutions allows to determine easily how far away any given solution is from the Pareto frontier. This highlights the importance of our findings and the relevance of the technology-agnostic reference frontier that we suggest in this paper.

\section{Discussion and conclusions}\label{sec:Conclusions}

In this paper we provide a characterization of the Pareto frontier capturing the trade-off between performance and fairness in binary prediction-based decision systems. Performance is measured by expected decision maker (DM) utility, while group fairness is quantified through a fairness score ($FS$) defined as a function of the expected decision subject (DS) utilities of the considered social groups.

Theorem \ref{thm: Pareto frontier} proves that the Pareto frontier consists of group-specific threshold rules applied to the probabilities $P[Y=1|\vec x]$ which might include upper-bound threshold rules. With this we generalize the results of \cite{hardt_equality_2016,corbett-davies_algorithmic_2017} for the cases of Demographic Parity and misclassification by enabling arbitrary fairness notions and arbitrary principles of justice (beyond egalitarianism). The proof of Theorem \ref{thm: threshold decision rule} shows that, for any DS utility matrix with $\tilde p^*=-(v_{10}-v_{00}) /(v_{11}-v_{10}+v_{00}-v_{01})\notin[0,1]$, Pareto-optimal solutions must be lower-bound threshold rules, confirming \cite{hardt_equality_2016,corbett-davies_algorithmic_2017}. However, if $\tilde p^*\in[0,1]$, both lower- and upper-bound threshold rules might turn out to be be Pareto optimal (as shown with an example in Section \ref{sec: illustrative example} Fig. \ref{fig: synth data examples}). Similarly, the results for $J=D$ generalizes the work of \cite{baumann_enforcing_2022} (showing that ub threshold might occur when enforcing of PPV or FOR parity) to more generalized DS utility matrices. 

The location of the Pareto frontier can be determined with any Bayes-optimal predictor of $P[Y=1|\vec x]$, and approximated with near Bayes-optimal predictors. However, an important insight of our paper is that the location of the Pareto-frontier only depends on the probability distributions $g(p|a)$ of the groups as demonstrated in Section \ref{sec: illustrative example}, thus on relatively simple characteristics of the population: It is not necessary to learn the full dependency structure of $\vec x$ and $Y$, it suffices to solve the much simpler problem of deriving $g(p)$ from the learning data. This observation opens up the field for new and improved ML methods that are trained to find the probability distributions $g(p|a)$ instead of the correct classification $\hat Y$ or prediction $\hat p$. Note that implementing a Pareto-optimal decision system still requires to have a good predictor (ideally Bayes-optimal) for $P[Y=1|\vec x]$, but for determining the location of the Pareto frontier, e.g. for assessing a given system and comparing it with the theoretically optimal system, this is not necessary.  

Note that the Pareto frontier resulting from Theorem \ref{thm: Pareto frontier} represents a fundamental limit: No decision system based on the same feature vector $\vec x$ can achieve better solutions. In other words: Any ML approach, whether based on pre-, in-, or postprocessing, trying to identify Pareto-optimal solutions will tend to replicate the same decision behavior, as long as the structure of the used decision model allows for group-specific threshold rules. Therefore, our results provide an implementation-agnostic benchmark which can be used to assess algorithms like \cite{liu_2022_accuracy,valdivia_2021_fair,zhang_2022_mitigating,wei_2022_fairness}. In Section \ref{sec: comparative study} (Fig.\ref{fig: section 5}(a)) we demonstrated this for the PF-SMG algorithm \cite{liu_2022_accuracy}.   

Contrary to in-processing methods for finding the Pareto frontier such as  \cite{liu_2022_accuracy,valdivia_2021_fair,zhang_2022_mitigating,wei_2022_fairness}, our result not only specifies the location of the Pareto frontier, but also characterizes the the Pareto-optimal solutions with respect to their decision characteristics, thus allowing to assess their moral implications. One of the implications of our findings is that optimizing for Pareto-optimal tradeoffs between DM utility and fairness tends to produce group-dependent threshold rules, regardless of whether they are implemented explicitly through post-processing or implicitly (and unknowningly) through black-box in-processing algorithms. The example of Section \ref{sec: comparative study} demonstrates that group-dependent threshold rules may be implemented even though the algorithm has no access to the sensitive attribute $a$. 
Furthermore, as stated before, the Pareto frontier may involve upper-bound threshold rules, phenomena as cherry-picking or within-group unfairness \cite{favier_cherry_2025, baumann_enforcing_2022} may be expected independent on the learning method, even though they may be obscured by model opacity. Both insights might raise concerns given ongoing debates about the moral permissibility of such behavior \cite{weerts_2022_goodhartslaw, favier_cherry_2025}.

Our findings also have legal implications: If group-dependent thresholds are considered discriminatory \cite{pessach_2022_review}, but high-performing ML methods tend to converge to the same decision behavior (even without using $a$), this raises the question whether such implementations should be considered discriminatory as well, and what this implies for the distinction between direct and indirect discrimination. We hope that further research will elaborate on such questions.

It should be stressed that our paper also has some limitations. The most important one is that our framework and the optimality theorem is based on one single fairness score. It cannot be applied to the simultaneous use of multiple fairness scores, such as for criteria like Equal Odds (FPR and FNR parity) or sufficiency (FPR and FOR parity) \cite{barocas_fairness_2023}. Also, while our framework allows group-specific utility functions, it does not cover individualized utility functions. The respective generalization of our results in these directions is left for future work.

\section*{Generative AI usage statement}
Generative AI tools were used to support grammatical correction and to improve the fluency of the text. All sections were initially written by the authors, and the use of generative AI was limited to stylistic refinement while keeping the text as close to the original as possible. The authors take full responsibility for the content of the paper.

\begin{acks}
We thank our three anonymous reviewers for their useful feedback throughout the reviewing process. This work is supported by a grant of the \textit{Digitalisierungsinitiative der Zürcher Hochschulen}.  
\end{acks}

\bibliographystyle{ACM-Reference-Format}
\bibliography{references}

@article{favier_cherry_2025,
  title={Cherry on the cake: fairness is NOT an optimization problem},
  author={Favier, Marco and Calders, Toon},
  journal={Machine Learning},
  volume={114},
  number={7},
  pages={160},
  year={2025},
  publisher={Springer},
  url = {https://doi.org/10.1007/s10994-025-06792-3}
}

@inproceedings{laufer_LDA_2025, series={CSLAW ’25},
   title={What Constitutes a Less Discriminatory Algorithm?},
   url={http://dx.doi.org/10.1145/3709025.3712214},
   DOI={10.1145/3709025.3712214},
   booktitle={Proceedings of the Symposium on Computer Science and Law on ZZZ},
   publisher={ACM},
   author={Laufer, Benjamin and Raghavan, Manish and Barocas, Solon},
   year={2025},
   month=mar, pages={136–151},
   collection={CSLAW ’25} }

@article{black_LDA_2024,
   author = {Emily Black and John Logan Koepke and Pauline T Kim and Solon Barocas and Mingwei Hsu},
   journal = {The Georgetown Law Journal},
   volum = {113},
   pages = {53--120},
   title = {Less Discriminatory Algorithms},
   url = {https://civilrights.org/2014/02/27/civil-rights-principles-era-},
   year = {2024}
}

@article{hardt_equality_2016,
  author       = {Moritz Hardt and
                  Eric Price and
                  Nathan Srebro},
  title        = {Equality of Opportunity in Supervised Learning},
  journal      = {CoRR},
  volume       = {abs/1610.02413},
  year         = {2016},
  url          = {http://arxiv.org/abs/1610.02413},
  eprinttype    = {arXiv}
}

@inproceedings{corbett-davies_algorithmic_2017,
author = {Corbett-Davies, Sam and Pierson, Emma and Feller, Avi and Goel, Sharad and Huq, Aziz},
title = {Algorithmic Decision Making and the Cost of Fairness},
year = {2017},
isbn = {9781450348874},
publisher = {ACM},
url = {https://doi.org/10.1145/3097983.3098095},
doi = {10.1145/3097983.3098095},
booktitle = {Proceedings of the 23rd ACM SIGKDD International Conference on Knowledge Discovery and Data Mining},
pages = {797–806},
numpages = {10},
series = {KDD '17}
}

@article{hertweck_justice-based_2023,
      title={A Justice-Based Framework for the Analysis of Algorithmic Fairness-Utility Trade-Offs}, 
      author={Corinna Hertweck and Joachim Baumann and Michele Loi and Eleonora Viganò and Christoph Heitz},
      year={2023},
      archivePrefix={arXiv},
      url={https://arxiv.org/abs/2206.02891}
}

@inproceedings{baumann_enforcing_2022, series={FAccT ’22},
   title={Enforcing Group Fairness in Algorithmic Decision Making: Utility Maximization Under Sufficiency},
   url={http://dx.doi.org/10.1145/3531146.3534645},
   DOI={10.1145/3531146.3534645},
   booktitle={2022 ACM Conference on Fairness Accountability and Transparency},
   publisher={ACM},
   author={Baumann, Joachim and Hannák, Anikó and Heitz, Christoph},
   year={2022},
   month=jun, pages={2315–2326},
   collection={FAccT ’22} }

@inproceedings{verma_2018_fairness,
  title={Fairness definitions explained},
  author={Verma, Sahil and Rubin, Julia},
  booktitle={Proceedings of the international workshop on software fairness},
publisher = {ACM},
url = {https://doi.org/10.1145/3194770.3194776},
  pages={1--7},
  year={2018},
series = {FairWare '18}
}

@book{barocas_fairness_2023,
  title = {Fairness and Machine Learning: Limitations and Opportunities},
  author = {Solon Barocas and Moritz Hardt and Arvind Narayanan},
  publisher = {MIT Press},
  year = {2023}
}

@article{mitchell_algorithmic_2021,
  title={Algorithmic Fairness: Choices, Assumptions, and Definitions},
  author={Shira Mitchell and Eric Potash and Solon Barocas and Alexander D'Amour and Kristian Lum},
  journal={Annual Review of Statistics and Its Application},
  year={2021},
  url={https://api.semanticscholar.org/CorpusID:228893833}
}

@inproceedings{hertweck_group_2024,
author = {Hertweck, Corinna and Loi, Michele and Heitz, Christoph},
year = {2024},
month = {05},
pages = {189-196},
title = {Group Fairness Refocused: Assessing the Social Impact of ML Systems},
booktitle={2024 11th IEEE Swiss Conference on Data Science (SDS)},
doi = {10.1109/SDS60720.2024.00034}
}

@article{arneson_egalitarianism_2002,
author = {Arneson, Richard J.},
title = {Egalitarianism and the Undeserving Poor},
journal = {Journal of Political Philosophy},
volume = {5},
number = {4},
pages = {327-350},
doi = {https://doi.org/10.1111/1467-9760.00037},
year = {1997}
}

@article{hertweck_distributive_2024, 
title={What’s Distributive Justice Got to Do with It? Rethinking Algorithmic Fairness from a Perspective of Approximate Justice},
DOI={10.1609/aies.v7i1.31661}, 
number={1}, 
journal={Proceedings of the AAAI/ACM Conference on AI, Ethics, and Society}, 
author={Hertweck, Corinna and Heitz, Christoph and Loi, Michele}, 
year={2024}, 
month={Oct.}, 
pages={597-608},
doi = {https://doi.org/10.1609/aies.v7i1.31661},
series = {AIES'24}
}

@book{rawls_justice_1999,
  address = {Cambridge, MA and London, England},
  author = {John Rawls},
  publisher = {Harvard University Press},
  title = {A Theory of Justice: Revised Edition},
  year = {1999}
}

@article{frankfurt_equality_1987, 
author = {Harry Frankfurt}, 
doi = {10.1086/292913}, 
journal = {Ethics}, 
number = {1}, 
pages = {21--43}, 
publisher = {Cambridge University Press}, 
title = {Equality as a Moral Ideal}, 
volume = {98}, 
year = {1987}}

@incollection{holtug_prioritarianism_2007,
	author = {Nils Holtug},
	booktitle = {Egalitarianism: new essays on the nature and value of equality},
	editor = {Nils Holtug and Kasper Lippert{-}Rasmussen},
	pages = {125--156},
	publisher = {Clarendon Press},
	title = {Prioritarianism},
	year = {2007}
}

@inproceedings{fleisher_fair_2021,
author = {Fleisher, Will},
title = {What's Fair about Individual Fairness?},
year = {2021},
isbn = {9781450384735},
publisher = {ACM},
url = {https://doi.org/10.1145/3461702.3462621},
doi = {10.1145/3461702.3462621},
booktitle = {Proceedings of the 2021 AAAI/ACM Conference on AI, Ethics, and Society},
pages = {480–490},
numpages = {11},
series = {AIES '21}
}

@inproceedings{dwork_fairness_2012,
author = {Dwork, Cynthia and Hardt, Moritz and Pitassi, Toniann and Reingold, Omer and Zemel, Richard},
title = {Fairness through awareness},
year = {2012},
isbn = {9781450311151},
publisher = {ACM},
url = {https://doi.org/10.1145/2090236.2090255},
doi = {10.1145/2090236.2090255},
booktitle = {Proceedings of the 3rd Innovations in Theoretical Computer Science Conference},
pages = {214–226},
numpages = {13},
series = {ITCS '12}
}

@article{kozodoi_fairness_2022,
title = {Fairness in credit scoring: Assessment, implementation and profit implications},
journal = {European Journal of Operational Research},
volume = {297},
number = {3},
pages = {1083-1094},
year = {2022},
issn = {0377-2217},
doi = {https://doi.org/10.1016/j.ejor.2021.06.023},
url = {https://www.sciencedirect.com/science/article/pii/S0377221721005385},
author = {Nikita Kozodoi and Johannes Jacob and Stefan Lessmann}
}

@article{fuster_predictably_2017,
  author  = {Andreas Fuster and Paul Goldsmith-Pinkham and Tarun Ramadorai and Ansgar Walther},
  title   = {Predictably Unequal? The Effects of Machine Learning on Credit Markets},
  journal = {The Journal of Finance},
  year    = {2017},
  volume  = {77},
  number  = {1},
  pages   = {5--47},
  doi     = {10.1111/jofi.13090}
}

@article{pessach_2022_review,
author = {Pessach, Dana and Shmueli, Erez},
title = {A Review on Fairness in Machine Learning},
year = {2022},
issue_date = {March 2023},
publisher = {ACM},
volume = {55},
number = {3},
issn = {0360-0300},
doi = {https://doi.org/10.1145/3494672},
journal = {ACM Computing Surveys (CSUR)},
month = feb,
articleno = {51},
numpages = {44}
}

@article{friedler_2018_impossibility,
  author       = {Sorelle A. Friedler and
                  Carlos Scheidegger and
                  Suresh Venkatasubramanian},
  title        = {On the (im)possibility of fairness},
  journal      = {CoRR},
  volume       = {abs/1609.07236},
  year         = {2016},
  url          = {http://arxiv.org/abs/1609.07236},
  eprinttype    = {arXiv}
}

@article{weerts_2022_goodhartslaw,
  title={Are There Exceptions to Goodhart's Law? On the Moral Justification of Fairness-Aware Machine Learning},
  author={Weerts, Hilde and Royakkers, Lamb{\`e}r and Pechenizkiy, Mykola},
  journal={ACM Journal on Responsible Computing},
  year={2026},
  publisher={ACM},
volume = {3},
number = {1},
doi = {https://doi.org/10.1145/3771734}
}

@article{moscato_2021_benchmark,
  title={A benchmark of machine learning approaches for credit score prediction},
  author={Moscato, Vincenzo and Picariello, Antonio and Sperl{\'\i}, Giancarlo},
  journal={Expert Systems with Applications},
  volume={165},
  pages={113986},
  year={2021},
  publisher={Elsevier},
doi = {https://doi.org/10.1016/j.eswa.2020.113986}
}

@article{beam_2018_big,
  title={Big data and machine learning in health care},
  author={Beam, Andrew L and Kohane, Isaac S},
  journal={JAMA},
  volume={319},
  number={13},
  pages={1317--1318},
  year={2018},
  publisher={American Medical Association},
  doi = {10.1001/jama.2017.18391},
  url = {https://doi.org/10.1001/jama.2017.18391}
}

@article{habehh_2021_machine,
  title={Machine learning in healthcare},
  author={Habehh, Hafsa and Gohel, Suril},
  journal={Current genomics},
  volume={22},
  number={4},
  pages={291--300},
  year={2021},
  publisher={Bentham Science Publishers direct}
}

@inproceedings{shailaja_2018_machine,
  title={Machine learning in healthcare: A review},
  author={Shailaja, K and Seetharamulu, Banoth and Jabbar, MA},
  booktitle={2018 Second international conference on electronics, communication and aerospace technology (ICECA)},
  pages={910--914},
  year={2018},
  publisher={IEEE},
doi={10.1109/ICECA.2018.8474918}
}

@article{van_2021_digital,
  title={Digital welfare fraud detection and the Dutch SyRI judgment},
  author={Van Bekkum, Marvin and Borgesius, Frederik Zuiderveen},
  journal={European Journal of Social Security},
  volume={23},
  number={4},
  pages={323--340},
  year={2021},
  publisher={SAGE Publications Sage UK: London, England},
doi = {https://doi.org/10.1177/13882627211031257}
}

@article{lagioia_2023_algorithmic,
  title={Algorithmic fairness through group parities? The case of COMPAS-SAPMOC},
  author={Lagioia, Francesca and Rovatti, Riccardo and Sartor, Giovanni},
  journal={AI \& society},
  volume={38},
  number={2},
  pages={459--478},
  year={2023},
  publisher={Springer},
doi = {https://doi.org/10.1007/s00146-022-01441-y}
}

@article{caton_fairness_2020,
  author       = {Simon Caton and
                  Christian Haas},
  title        = {Fairness in Machine Learning: {A} Survey},
  journal      = {CoRR},
  volume       = {abs/2010.04053},
  year         = {2020},
  url          = {https://arxiv.org/abs/2010.04053},
  eprinttype    = {arXiv}
}

@article{wachter_2021_fairness,
  title={Why fairness cannot be automated: Bridging the gap between EU non-discrimination law and AI},
  author={Wachter, Sandra and Mittelstadt, Brent and Russell, Chris},
  journal={Computer Law \& Security Review},
  volume={41},
  pages={105567},
  year={2021},
  publisher={Elsevier},
doi = {https://doi.org/10.1016/j.clsr.2021.105567}
}

@inproceedings{weerts_2023_algorithmic,
  title={Algorithmic unfairness through the lens of EU non-discrimination law: Or why the law is not a decision tree},
  author={Weerts, Hilde and Xenidis, Rapha{\"e}le and Tarissan, Fabien and Olsen, Henrik Palmer and Pechenizkiy, Mykola},
  booktitle={Proceedings of the 2023 ACM Conference on Fairness, Accountability, and Transparency},
  pages={805--816},
  year={2023},
publisher = {ACM},
url = {https://doi.org/10.1145/3593013.3594044},
series = {FAccT '23}
}

@article{hu_2018_welfare,
  title={Welfare and distributional impacts of fair classification},
  author={Hu, Lily and Chen, Yiling},
archivePrefix={arXiv},
  year={2018},
url={https://arxiv.org/abs/1807.01134}
}

@article{liu_2022_accuracy,
  title={Accuracy and fairness trade-offs in machine learning: A stochastic multi-objective approach},
  author={Liu, Suyun and Vicente, Luis Nunes},
  journal={Computational Management Science},
  volume={19},
  number={3},
  pages={513--537},
  year={2022},
  publisher={Springer},
doi = {https://doi.org/10.1007/s10287-022-00425-z}
}

@article{tang_2023_theoretical,
  title={A theoretical approach to characterize the accuracy-fairness trade-off pareto frontier},
  author={Tang, Hua and Cheng, Lu and Liu, Ninghao and Du, Mengnan},
  year={2023},
archivePrefix={arXiv},
url={https://arxiv.org/abs/2310.12785}
}

@article{valdivia_2021_fair,
  title={How fair can we go in machine learning? Assessing the boundaries of accuracy and fairness},
  author={Valdivia, Ana and S{\'a}nchez-Monedero, Javier and Casillas, Jorge},
  journal={International Journal of Intelligent Systems},
  volume={36},
  number={4},
  pages={1619--1643},
  year={2021},
  publisher={Wiley Online Library},
doi = {https://doi.org/10.1002/int.22354}
}

@article{zhang_2022_mitigating,
  title={Mitigating unfairness via evolutionary multiobjective ensemble learning},
  author={Zhang, Qingquan and Liu, Jialin and Zhang, Zeqi and Wen, Junyi and Mao, Bifei and Yao, Xin},
  journal={IEEE transactions on evolutionary computation},
  volume={27},
  number={4},
  pages={848--862},
  year={2022},
  publisher={IEEE},
doi={10.1109/TEVC.2022.3209544}
}

@article{wei_2022_fairness,
  title={The fairness-accuracy Pareto front},
  author={Wei, Susan and Niethammer, Marc},
  journal={Statistical Analysis and Data Mining: The ASA Data Science Journal},
  volume={15},
  number={3},
  pages={287--302},
  year={2022},
  publisher={Wiley Online Library},
doi = { https://doi.org/10.1002/sam.11560}
}

@inproceedings{silvia_2020_general,
  title={A general approach to fairness with optimal transport},
  author={Silvia, Chiappa and Ray, Jiang and Tom, Stepleton and Aldo, Pacchiano and Heinrich, Jiang and John, Aslanides},
  booktitle={Proceedings of the AAAI Conference on Artificial Intelligence},
  volume={34},
  number={04},
  pages={3633--3640},
  year={2020},
doi = {https://doi.org/10.1609/aaai.v34i04.5771}
}

@inproceedings{jang_2022_group,
  title={Group-aware threshold adaptation for fair classification},
  author={Jang, Taeuk and Shi, Pengyi and Wang, Xiaoqian},
  booktitle={Proceedings of the AAAI Conference on Artificial Intelligence},
  volume={36},
  number={6},
  pages={6988--6995},
  year={2022},
doi = {https://doi.org/10.1609/aaai.v36i6.20657}
}

@article{calders_2010_three,
  title={Three naive bayes approaches for discrimination-free classification},
  author={Calders, Toon and Verwer, Sicco},
  journal={Data mining and knowledge discovery},
  volume={21},
  number={2},
  pages={277--292},
  year={2010},
  publisher={Springer},
  doi = {https://doi.org/10.1007/s10618-010-0190-x}
}

@misc{baumann_2023_distributivejusticefoundationalpremise,
      title={Distributive Justice as the Foundational Premise of Fair ML: Unification, Extension, and Interpretation of Group Fairness Metrics}, 
      author={Joachim Baumann and Corinna Hertweck and Michele Loi and Christoph Heitz},
      year={2023},
      eprint={2206.02897},
      archivePrefix={arXiv},
      primaryClass={cs.CY},
      url={https://arxiv.org/abs/2206.02897}, 
}

@misc{cen_2025_auditsresourcedataaccess,
      title={Audits Under Resource, Data, and Access Constraints: Scaling Laws For Less Discriminatory Alternatives}, 
      author={Sarah H. Cen and Salil Goyal and Zaynah Javed and Ananya Karthik and Percy Liang and Daniel E. Ho},
      year={2025},
      eprint={2509.05627},
      archivePrefix={arXiv},
      primaryClass={cs.CY},
      url={https://arxiv.org/abs/2509.05627}, 
}

\appendix

\section{Utility-based fairness evaluation} \label{app: utility fairness}
In this section, we show that the classical confusion-based metrics are special cases of the expectation value $E[V|J=j]$ with specific choices of $V,J,j$ (see also \cite{baumann_2023_distributivejusticefoundationalpremise}). For notational convenience, we drop the class index $a$ in the following.

\

\noindent \textbf{Decision rate $P[D=1]$} 

The decision rate is given by $P[D=1] = p_{10}+p_{11}$. 
The unconditional DS subject utility is given by 
\begin{equation*}
    E[V]=v_{00}p_{00}+v_{01}p_{01}+v_{10}p_{10}+v_{11}p_{11} 
\end{equation*}

Both quantities are equal if $v_{00}=v_{01}=0$ and $v_{10}=v_{11}=1$. Thus, $P[D=1] = E[V|J=j]$ for the choices $(v_{ij})=(0,0;1,1)$ and $J=\emptyset$.

\noindent \textbf{True Positive Rate (TPR), False Positive Rate (FPR), True Negative Rate (TNR), False Negative Rate (FNR)} 

We start with TPR, given by the conditional probability $P[D=1|Y=1]$, which can be expressed as 
\begin{equation*}
    P[D=1|Y=1]=\frac{P[D=1,Y=1]}{P[Y=1]} =\frac{p_{11}}{p_{01}+p_{11}}
\end{equation*}

The expected decision subject utility conditioned on $Y=1$ is given by 
\begin{equation*}
    E[V|Y=1]=\frac{E[V,Y=1]}{E[Y=1]}=\frac{v_{01}p_{01}+v_{11}p_{11}}{p_{01}+p_{11}}
\end{equation*}

Both quantities are equal if $v_{01}=0$ and $v_{11}=1$. Similarly, it is easy to show that FPR $P[D=1|Y=0]$ equals $E[V|Y=0]$ if $v_{00}=0$ and $v_{10}=1$. 

Thus, with a choice of $(v_{ij})=(0,0;1,1)$ (same DS utility matrix as for the case of the selection rate) and $J=Y$, both metrics can be met: $P[D=1|Y=y]$ equals $E(V|J=j)$ for the choice $(v_{ij})=(0,0;1,1)$, $J=Y$, and $j=y$.

Analogously, TNR and FNR DS utility matrix $(v_{ij})=(0,0;1,1)$ needs to be chosen. 

\

\noindent \textbf{Positive Predictive Value (PPV), False Omission Rate (FOR), Negative Predictive Parity (NPR), False Discovery Rate (FDR) } 

The confusion matrix metrics conditioned on the value of $D$ (e.g. $PPV=P[Y=1|D=1]$) are derived equivalently. 

For both PPV and FOR, the DS utility matrix $(v_{ij})=(0, 1; 0, 1)$ can be used, and $P[Y=1|D=d] = E[V|J=j]$ for a choice of $(v_{ij})=(0, 1; 0, 1)$, $J=D$, and $j=d$. For NPR and FDR, the DS utility matrix is $(v_{ij})=(1, 0; 1, 0)$.

\section{Proof of Theorem \ref{thm: threshold decision rule} (unconditional case)} \label{app: proof theorem threshold rule}
In this Section we proof Theorem \ref{thm: threshold decision rule} for the unconditional case $E[V|a]=\eta$ (i.e. $J=\emptyset$). The  conditional case $E[V|J,a]=\eta$ with $J=\{Y,D\}$ is discussed in Appendix \ref{app: addendum proof 3.1}. 

\begin{proof}
    Let $d\in\mathcal D_\eta $ be an arbitrary decision rule with $E_{d}[V|a]=\eta$ and $E_d[U|a]=U_0$. We consider the discretized version of the decision rule $d=(d_1,...,d_N)$, making the limit $N\rightarrow \infty$ at the end of the proof. A lower-bound threshold rule is characterized by $d_i=0$ for $i<i_t$, $d_i=1$ for $i>i_t$ and $d_{i_t}\in[0,1]$, and an upper-bound threshold rule by $d_i=1$ for $i<i_t$, $d_i=0$ for $i>i_t$ and $d_{i_t}\in [0,1]$. 

    We now consider a small change ($\Delta_j,\Delta_k$) of ($d_j,d_k$), i.e. $d_j\mapsto d_j+\Delta_j$ and $ d_k\mapsto d_k+\Delta_k$. From Equations (\ref{eq: riemann sum E(U)}) and (\ref{eq: riemann sum E(V)}) we derive that a change in decision $d_j$ only affects $E[U]$ via the term $\alpha p_j+\beta$ and $E[V]$ via $\tilde{\alpha}p_j+\tilde{\beta}$. So, a small change $(\Delta_j,\Delta_k$) of ($d_j,d_k$) changes $E[U]$ by $\Delta E[U|a]=\Delta_jf(p_j)w_j+\Delta_kf(p_k)w_j$ and $E[V|a]$ by $\Delta E[V|a]=\Delta_j\tilde f( p_j)w_j+\Delta_k\tilde f( p_k)w_k$, with $f(p)=\alpha\cdot p+\beta$ and $\tilde{f}(p)=\tilde{\alpha}\cdot p+\tilde{\beta}$. In the following we write $p^*$ and $\tilde p^*$ for the intersection points of $f(p)$ and $\tilde f(p)$ with the $x$-axis, i.e. $f(p^*)=0$ and $\tilde f(\tilde p^*)=0$.
    
    We consider changes $(\Delta_j,\Delta_k)$ such that $E[V|a]=\eta$ does not change, i.e. $\Delta E[V|a]=0$. This can be done by choosing a combination ($\Delta_j,\Delta_k$) such that the following condition holds
    \begin{equation}\label{eq:delta dj}
        \Delta_j = -\frac{\tilde{f}(p_k)w_k}{\tilde{f}(p_j)w_j}\Delta_k.
    \end{equation}
    Imposing this condition on the simultaneous change $(\Delta_j,\Delta_k)$ results in a change in $E[U|a]$ given by
    \begin{equation}\label{eq: delta Udm}
        \begin{split}
        \Delta E[U|a] & = \Delta_j f(p_j)w_j+\Delta_k f(p_k)w_k \\
         & =\Delta_k  w_k\tilde{f}(p_k)\Big(f(p_k)/\tilde{f}(p_k)-f(p_j)/\tilde{f}(p_j)\Big) \\
         & = \Delta_k  w_k\tilde{f}(p_k)\Big(\mathcal{F}(p_k)-\mathcal{F}(p_j)\Big)
        \end{split}
    \end{equation}
    with $\mathcal{F}(p)=f(p)/\tilde{f}(p)$. By choosing the appropriate sign for $\Delta_k$, we can increase the DM utility (i.e. $\Delta E[U|a]>0$). 
    
    Note that the sign we should choose for $\Delta_k$ in order to achieve $\Delta E[U|a]>0$ depends on the shapes of $\tilde f(p)$ and $\mathcal F(p)$. The shape of $f(p)$ is fixed on the interval $[0,1]$ in the sense that $f(0)<0$, $f(1)>0$ and the intersection point $p^*=-\beta/\alpha\in[0,1]$ (cmp. conditions stated in Section \ref{subsec: utility-based}). Figures \ref{fig: plots f functions}(a)-(d) (Appendix \ref{app: plots f-functions}) show the four possible shapes that $\tilde{f}(p)$ can have with respect to $f(p)$ and Figures \ref{fig: plots f functions}(e)-(f) shows for each of those scenario's the shape of $\mathcal F(p)=f(p)/\tilde{f}(p)$.

    \textbf{In a first case}, we assume $\tilde p^*\not \in[0,1]$. Note that it always holds that $p^*\in[0,1]$ due to the fact that $\alpha>0$ and $\beta<0$ (see Section \ref{subsec: utility-based}). We consider $(j,k)$ with $p_j<p_k$. 
    \begin{itemize}
        \item[(a)] $\tilde{a}>0$ and $p^*>\tilde p^*$ (Fig. \ref{fig: plots f functions}(a)) : From $\tilde p^*<p^*$, $p^*<1$ and $\tilde p^*\not \in[0,1]$ it follows that $\tilde p^*<0$. With $\tilde{a}>0$ it follows that $\tilde f(p)>0$ for all $p\in[0,1]$, so $\tilde f(p_k)>0$. Furthermore, $\big(\mathcal F(p_k)-\mathcal F(p_j)\big) > 0$ (Fig. \ref{fig: plots f functions}(e) right of vertical asymptote) for all possible $p_j,p_k\in[0,1]$ with $p_j<p_k$. Therefore, $\Delta E[U|a]>0$ requires $\Delta_k>0$ according to Equation (\ref{eq: delta Udm}), leading to $\Delta _j<0$ (Equation (\ref{eq:delta dj})).
        
        \item[(b)] $\tilde{a}<0$ and $p^*<\tilde p^*$ (Fig. \ref{fig: plots f functions}(b)): From $\tilde p^*>p^*$, $p^*>0$ and $\tilde p^*\notin[0,1]$ it follows that $\tilde p^*>1$. With $\tilde\alpha<0$ it follows that $\tilde f(p)>0$ for all $p\in[0,1]$, so $\tilde f(p_k)>0$. Furthermore, $\big(\mathcal F(p_k)-\mathcal F(p_j)\big) > 0$ (Fig. \ref{fig: plots f functions}(e) left of vertical asymptote) for all possible $p_j,p_k\in[0,1]$ with $p_j<p_k$. Therefore, $\Delta E[U|a]>0$ requires $\Delta_k>0$ according to Equation (\ref{eq: delta Udm}), leading to $\Delta _j<0$ (Equation (\ref{eq:delta dj})).
        
        \item[(c)] $\tilde{a}>0$ and $p^*<\tilde p^*$ (Fig. \ref{fig: plots f functions}(c)): From $\tilde p^*>p^*$, $p^*>0$ and $\tilde p^*\notin[0,1]$ it follows that $\tilde p^*>1$. With $\tilde\alpha>0$ it follows that $\tilde f(p)<0$ for all $p\in[0,1]$, so $\tilde f(p_k)<0$. Furthermore, $\big(\mathcal F(p_k)-\mathcal F(p_j)\big) < 0$ (Fig. \ref{fig: plots f functions}(f) left of vertical asymptote) for all possible $p_j,p_k\in[0,1]$ with $p_j<p_k$. Therefore, $\Delta E[U|a]>0$ requires $\Delta_k>0$ according to Equation (\ref{eq: delta Udm}), leading to $\Delta _j<0$ (Equation (\ref{eq:delta dj})).
        
        \item[(d)] $\tilde{a}<0$ and $p^*>\tilde p^*$ (Fig. \ref{fig: plots f functions}(d)): From $\tilde p^*< p^*$, $p^*<1$ and $\tilde p^*\notin [01,]$ it follows that $\tilde p^*<0$. With $\tilde\alpha<0$ it follows that $\tilde f(p)<0$ for all $p\in[0,1]$, so $\tilde f(p_k)<0$. Furthermore, $\big(\mathcal F(p_k)-\mathcal F(p_j)\big) < 0$ (Fig. \ref{fig: plots f functions}(f) right of vertical asymptote) for all possible $p_j,p_k\in[0,1]$ with $p_j<p_k$. Therefore, $\Delta E[U|a]>0$ requires $\Delta_k>0$ according to Equation (\ref{eq: delta Udm}), leading to $\Delta _j<0$ (Equation (\ref{eq:delta dj})).
    \end{itemize}
    Note that all possible scenarios lead to the same outcome: any given decision rule $d(p)$ can be improved with respect to $E[U|a]$ (without changing $E[V|a]$), if there is at least one  combination $(d_j,d_k)$ with $d_j>0$ and $d_k<1$. Notice that unless $d(p)$ is a lower-bound threshold rule as defined above, there always exists such a combination. Thus: For $\tilde p^*\notin[0,1]$, the decision rule maximizing $E[U|a]$ must be a lower-bound threshold rule. 

    \textbf{In a second case}, we assume $\tilde p^* \in[0,1]$, again considering $(j,k)$ with $p_j<p_k$. Let $i^*$ be the index of the bin which contains $\tilde p^*$. Again, four different scenarios need to be distinguished: 
    \begin{itemize}
        \item[(a)] $\tilde{a}>0$ and $p^*>\tilde p^*$ (Fig. \ref{fig: plots f functions}(a)).  
        \begin{itemize}
            \item[(a1)] $j<i^*<k$: From $\tilde{a}>0$ it follows that $\tilde f(p_j)<0$ and $\tilde f(p_k)>0$. Fig. \ref{fig: plots f functions}(e) shows that $\big(\mathcal F(p_k)-\mathcal F(p_j)\big)<0$. Thus, $\Delta_k<0$ is required for achieving $\Delta E[U|a]>0$ (Eq. (\ref{eq: delta Udm})), which leads to $\Delta_j<0$ (Eq. \ref{eq:delta dj}). So, $\Delta E[U|a]>0$ can be achieved for any combination $(d_j,d_k)$ with $d_j>0$ and $d_k>0$. This is always possible unless either $d_j=0$ for all $j<i^*$, or $d_j=0$ for all $j>i^*$. 
            
            \item[(a2)] $j<k<i^*$: $\tilde f(p_j)<0$, $\tilde f(p_k)<0$ and $\big(\mathcal F(p_k)-\mathcal F(p_j)\big)>0$ (Fig. \ref{fig: plots f functions}(e) left of vertical asymptote). $\Delta E[U|a]>0$ requires $\Delta_k<0$ (Eq. (\ref{eq: delta Udm})) and $\Delta_j>0$ (Eq. (\ref{eq:delta dj})). So, $\Delta E[U|a]>0$ can be achieved for any combination $(d_j,d_k)$ (with $j<k<i^*$) for which $d_j<1$ and $d_k>0$. Such a combination always exists unless the decision vector $(d_1,d_2,...,d_{i^*-1})$ is an upper-bound threshold rule.

            \item[(a3)] $i^*<j<k$: $\tilde f(p_j)>0$, $\tilde f(p_k)>0$ and $\big(\mathcal F(p_k)-\mathcal F(p_j)\big)>0$ (Fig. \ref{fig: plots f functions}(e) right of vertical asymptote). $\Delta E[U|a]>0$ requires $\Delta_k>0$ (Eq. (\ref{eq: delta Udm})) and $\Delta_j<0$ (Eq. (\ref{eq:delta dj})). So, $\Delta E[U|a]>0$ can be achieved for any combination $(d_j,d_k)$ (with $i^*<j<k<$) for which $d_j>0$ and $d_k<1$. Such a combination always exists unless the decision vector $(d_{i+1},...,d_N)$ is a lower-bound threshold rule. 
        \end{itemize}
        From the first point we know that the utility maximizing decision rule $d(p)$ is either 0 for all $p<\tilde p^*$ or 0 for all $p>\tilde p^*$. If $d(p)=0$ for all $p>\tilde{p}^*$ the second point tels us that on the interval $[0,\tilde p^*)$ the utility maximizing decision rule is an upper-bound threshold rule, which means that on $[0,1]$ the utility maximizing decision rule is an upper-bound threshold rule with threshold $t\in[0,\tilde p^*)$. On the other hand, if $d(p)=0$ for all $p<\tilde p^*$ the third point tels us that on the interval $(\tilde p^*,1]$ the utility maximizing decision rule is a lower-bound threshold rule, which means that on $[0,1]$ the utility maximizing decision rule is a lower-bound threshold rule with threshold $t\in(\tilde p^*,1]$.   

        \item[(b)] $\tilde{a}<0$ and $p^*>\tilde p^*$ (Fig. \ref{fig: plots f functions}(b)). 
        \begin{itemize}
            \item[(b1)] $j<i^*<k$: From $\tilde \alpha<0$ it follows that $\tilde f(p_j)>0$ and $\tilde f(p_k)<0$. Fig. \ref{fig: plots f functions}(e) shows that $\big(\mathcal F(p_k)-\mathcal f(p_j)\big)<0$. Thus $\Delta _k>0$ is required for achieving $\Delta E[U|a]>0$ (Eq. (\ref{eq: delta Udm})), which leads to $\Delta_j>0$ (Eq. (\ref{eq:delta dj})). So, $\Delta E[U|a]>0$ can be achieved for any combination $(d_j,d_k)$ with $d_j<1$ and $d_k<1$. This is always possible unless either $d_j=1$ for all $j<i^*$, or $d_j=1$ for all $j>i^*$.

            \item[(b2)] $j<k<i^*$: $\tilde f(p_j)>0$, $\tilde f(p_k)>0$ and $\big(\mathcal F(p_k)-\mathcal F(p_j)\big)>0$ (Fig. \ref{fig: plots f functions}(e) left of vertical asymptote). $\Delta E[U|a]>0$ requires $\Delta_k>0$ (Eq. (\ref{eq: delta Udm})) and $\Delta_j<0$ (Eq. (\ref{eq:delta dj})). So, $\Delta E[U|a]>0$ can be achieved for any combination $(d_j,d_k)$ (with $j<k<i^*$) for which $d_j>0$ and $d_k<1$. Such a combination always exists unless the decision vector $(d_1,d_2,...,d_{i*-1})$ is a lower-bound threshold rule.

            \item [(b3)] $i^*<j<k$: $\tilde f(p_j)<0$, $\tilde f(p_k)<0$ and $\big(\mathcal F(p_k)-\mathcal F(p_j)\big)>0$ (Fig. \ref{fig: plots f functions}(e) right of vertical asymptote). $\Delta E[U|a]>0$ requires $\Delta_k<0$ (Eq. (\ref{eq: delta Udm})) and $\Delta_j>0$ (Eq. (\ref{eq:delta dj})). So, $\Delta E[U|a]>0$ can be achieved for any combination $(d_j,d_k)$ (with $j<k<i^*$) for which $d_j<1$ and $d_k>0$. Such a combination always exists unless the decision vector $(d_{i+1},...,d_N)$ is an upper-bound threshold rule.

        \end{itemize}
        Combining these pieces of information gives us that the utility maximizing threshold rule is either an upper-bound threshold rule with threshold $t\in(\tilde p^*,1]$, or a lower-bound threshold rule with threshold $t\in[0,\tilde p^*)$.

        \item[(c)] $\tilde{a}>0$ and $p^*<\tilde p^*$ (Fig. \ref{fig: plots f functions}(c)).
        \begin{itemize}
            \item[(c1)] $j<i^*<k$: From $\tilde \alpha>0$ it follows that $\tilde f(p_j)<0$ and $\tilde f(p_k)>0$. Fig. \ref{fig: plots f functions}(f) shows that $\big(\mathcal F(p_k)-\mathcal f(p_j)\big)>0$. Thus $\Delta _k>0$ is required for achieving $\Delta E[U|a]>0$ (Eq. (\ref{eq: delta Udm})), which leads to $\Delta_j>0$ (Eq. (\ref{eq:delta dj})). So, $\Delta E[U|a]>0$ can be achieved for any combination $(d_j,d_k)$ with $d_j<1$ and $d_k<1$. This is always possible unless either $d_j=1$ for all $j<i^*$, or $d_j=1$ for all $j>i^*$.
            
            \item[(c2)] $j<k<i^*$: $\tilde f(p_j)<0$, $\tilde f(p_k)<0$ and $\big(\mathcal F(p_k)-\mathcal F(p_j)\big)<0$ (Fig. \ref{fig: plots f functions}(f) left of vertical asymptote). $\Delta E[U|a]>0$ requires $\Delta_k>0$ (Eq. (\ref{eq: delta Udm})) and $\Delta_j<0$ (Eq. (\ref{eq:delta dj})). So, $\Delta E[U|a]>0$ can be achieved for any combination $(d_j,d_k)$ (with $j<k<i^*$) for which $d_j>0$ and $d_k<1$. Such a combination always exists unless the decision vector $(d_1,d_2,...,d_{i*-1})$ is a lower-bound threshold rule.

            \item [(c3)] $i^*<j<k$: $\tilde f(p_j)>0$, $\tilde f(p_k)>0$ and $\big(\mathcal F(p_k)-\mathcal F(p_j)\big)<0$ (Fig. \ref{fig: plots f functions}(f) right of vertical asymptote). $\Delta E[U|a]>0$ requires $\Delta_k<0$ (Eq. (\ref{eq: delta Udm})) and $\Delta_j>0$ (Eq. (\ref{eq:delta dj})). So, $\Delta E[U|a]>0$ can be achieved for any combination $(d_j,d_k)$ (with $j<k<i^*$) for which $d_j<1$ and $d_k>0$. Such a combination always exists unless the decision vector $(d_{i+1},...,d_N)$ is an upper-bound threshold rule.
            
        \end{itemize}
        Combining these pieces of information gives us that the utility maximizing threshold rule is either an upper-bound threshold rule with threshold $t\in(\tilde p^*,1]$, or a lower-bound threshold rule with threshold $t\in[0,\tilde p^*)$.

        \item[(d)] $\tilde{a}<0$ and $p^*>\tilde p^*$ (Fig. \ref{fig: plots f functions}(d)).
        \begin{itemize}
            \item[(d1)] $j<i^*<k$: From $\tilde{a}<0$ it follows that $\tilde f(p_j)>0$ and $\tilde f(p_k)<0$. Fig. \ref{fig: plots f functions}(f) shows that $\big(\mathcal F(p_k)-\mathcal F(p_j)\big)>0$. Thus, $\Delta_k<0$ is required for achieving $\Delta E[U|a]>0$ (Eq. (\ref{eq: delta Udm})), which leads to $\Delta_j<0$ (Eq. \ref{eq:delta dj}). So, $\Delta E[U|a]>0$ can be achieved for any combination $(d_j,d_k)$ with $d_j>0$ and $d_k>0$. This is always possible unless either $d_j=0$ for all $j<i^*$, or $d_j=0$ for all $j>i^*$.    
            
            \item[(d2)] $j<k<i^*$: $\tilde f(p_j)>0$, $\tilde f(p_k)>0$ and $\big(\mathcal F(p_k)-\mathcal F(p_j)\big)<0$ (Fig. \ref{fig: plots f functions}(f) left of vertical asymptote). $\Delta E[U|a]>0$ requires $\Delta_k<0$ (Eq. (\ref{eq: delta Udm})) and $\Delta_j>0$ (Eq. (\ref{eq:delta dj})). So, $\Delta E[U|a]>0$ can be achieved for any combination $(d_j,d_k)$ (with $j<k<i^*$) for which $d_j<1$ and $d_k>0$. Such a combination always exists unless the decision vector $(d_1,d_2,...,d_{i^*-1})$ is an upper-bound threshold rule.

            \item[(d3)] $i^*<j<k$: $\tilde f(p_j)<0$, $\tilde f(p_k)<0$ and $\big(\mathcal F(p_k)-\mathcal F(p_j)\big)<0$ (Fig. \ref{fig: plots f functions}(f) right of vertical asymptote). $\Delta E[U|a]>0$ requires $\Delta_k>0$ (Eq. (\ref{eq: delta Udm})) and $\Delta_j<0$ (Eq. (\ref{eq:delta dj})). So, $\Delta E[U|a]>0$ can be achieved for any combination $(d_j,d_k)$ (with $i^*<j<k<$) for which $d_j>0$ and $d_k<1$. Such a combination always exists unless the decision vector $(d_{i+1},...,d_N)$ is a lower-bound threshold rule.
        \end{itemize}
        Combining these pieces of information gives us that the utility maximizing threshold rule is either an upper-bound threshold rule with threshold $t\in[0,\tilde p^*)$, or a lower-bound threshold rule with threshold $t\in(\tilde p^*,1]$.
    \end{itemize}
    So in the case where $\tilde p^*\in[0,1]$ all possible scenario's lead to the outcome that the utility optimizing decision rule $d(p)$ is either given by a lower-bound threshold rule, or by an upper-bound threshold rule.

    From the results of case 1 and 2 we know that the decision rule that maximizes the Riemann approximation of $E[U]$ for a given DS utility is given by a lower-bound threshold rule in the case that the elements of $(v_{ij})$ are defined such that $\tilde f(p)$ does not intersect the $x$-axis on the interval $[0,1]$, and by either a lower- or upper-bound threshold rule when $\tilde f(p)$ does intersect the $x$-axis on the interval $[0,1]$.
    By taking the limit $N\to \infty$ we generalize this result to also hold for the exact calculation of $E[U]$.
    
\end{proof}

\section{Complementary figure to proof of theorem \ref{thm: threshold decision rule}}\label{app: plots f-functions}
Figures \ref{fig: plots f functions}(a)-(d) show the different shapes that $f(p)$ and $\tilde f(p)$ can take with respect to each other: $\tilde{f}(p)$ can have a positive or negative slope and $p^*$ can be larger or smaller than $\tilde p^*$. The horizontal line in the plots represents the $x-$axis, on which the intersection points $p^*=-\beta/\alpha$ and $\tilde p^*=-\tilde \beta/\tilde\alpha$ are plotted. From the constraints we have on $f(p)$ (that is $\alpha>0$, $\beta>0$ and $\alpha+\beta>0$ as pointed out in Section \ref{subsec: utility-based}), we know that $p^*$ lies in the interval $[0,1]$. As we do not have those constraints on $\tilde f(p)$, $\tilde p^*=-\tilde\beta/\tilde\alpha$ could lie both inside and outside of the interval $[0,1]$.

\begin{figure}[h!]
    \includegraphics[width=0.24\linewidth]{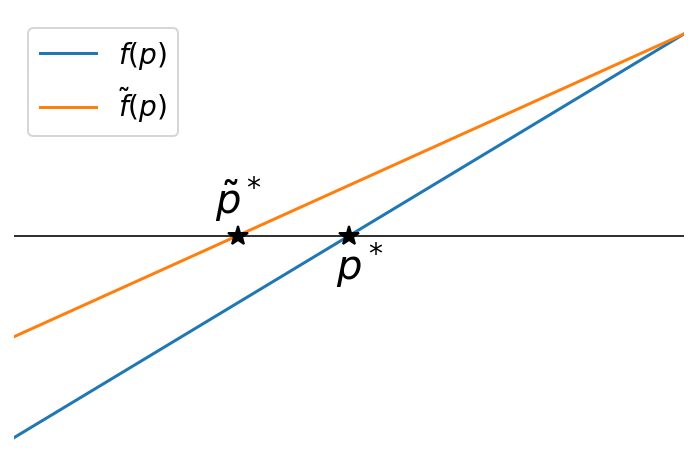}
    \includegraphics[width=0.24\linewidth]{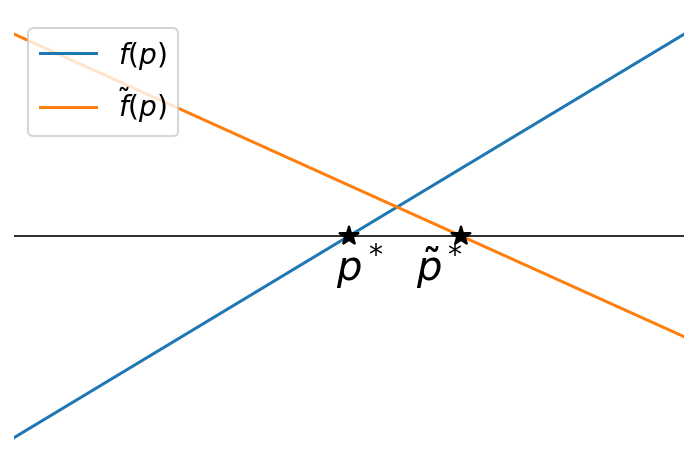}
    \includegraphics[width=0.24\linewidth]{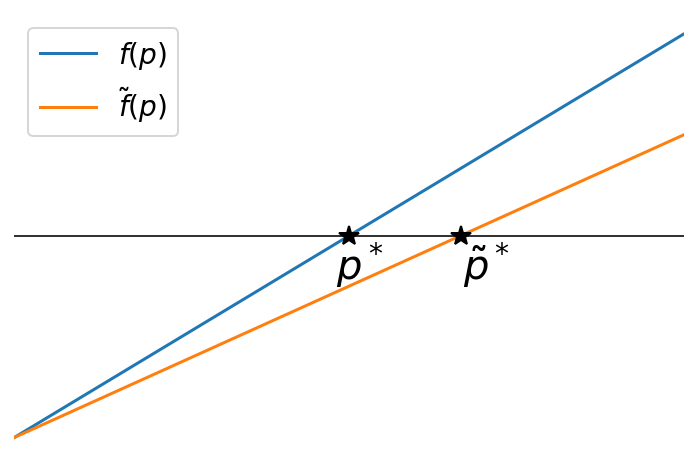}
    \includegraphics[width=0.24\linewidth]{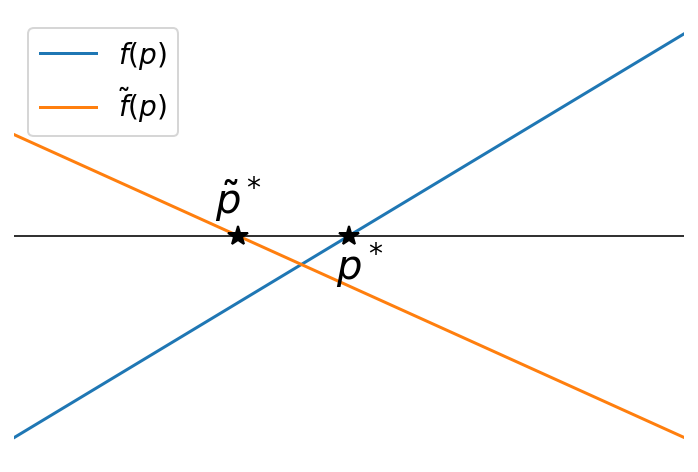}
    \\ 
    \small (a) $\tilde{\alpha}>0$ and $p^*>\tilde p^*$ \quad\quad\quad
    (b) $\tilde \alpha <0$ and $p^*<\tilde p^*$ \quad\quad\quad
    (c) $\tilde \alpha >0$ and $p^*<\tilde p^*$ \quad\quad\quad
    (d) $\tilde \alpha <0$ and $p^*>\tilde p^*$
    \\
    \includegraphics[width=0.35\linewidth]{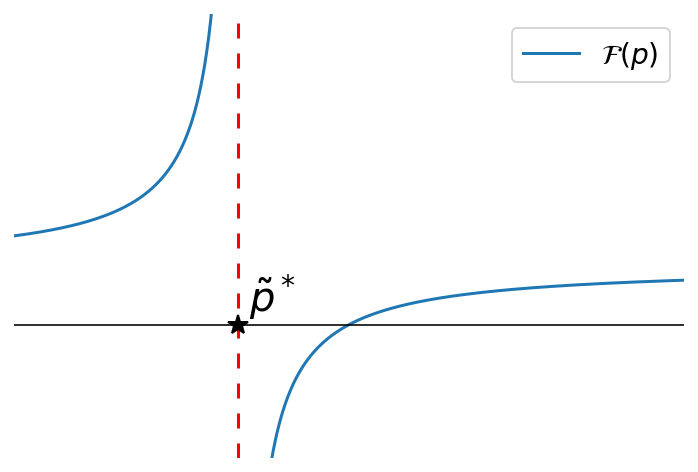}\quad\quad\quad\quad\quad\quad
    \includegraphics[width=0.35\linewidth]{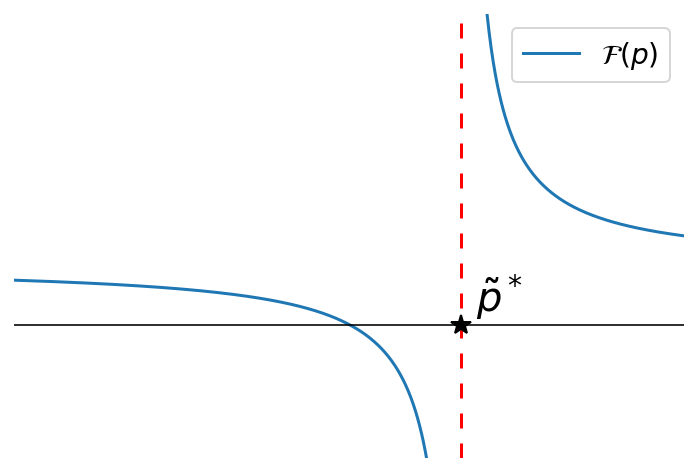}
    \\
    \small (e) $\mathcal F(p)$ for (a) and (b) \quad\quad\quad\quad\quad\quad\quad\quad\quad\quad\quad\quad
    (f) $\mathcal F(p)$ for (c) and (d) 

    \caption{(a)-(d) show the four possible shapes of $\tilde f(p)$ with respect to $f(p)$. (e) shows the function $\mathcal F(p)$ for the scenario's (a) en (b). (f) shows the function $\mathcal F(p)$ for the scenario's (c) and (d).}
    \label{fig: plots f functions}
\end{figure}

Figures \ref{fig: plots f functions}(e)-(f) show the two shapes that $\mathcal{F}(p)=f(p)/\tilde{f}(p)$ can take for different combinations of $f(p)$ and $\tilde f(p)$. The vertical asymptote (red dotted line) is given by $\tilde{p}^*=-\tilde{\beta}/\tilde{\alpha}$ and the horizontal asymptote is given by $\alpha/\tilde{\alpha}$.

\section{Addendum proof theorem \ref{thm: threshold decision rule} for $J\neq \emptyset$} \label{app: addendum proof 3.1}
Some popular fairness criteria are based on expected DS utility conditioned on a justifier $J$, where $J=Y$ (e.g. \textit{true positive rate}) or $J=D$ (e.g. \textit{positive predictive value}). In the following, we proof Theorem \ref{thm: threshold decision rule} for two of these situations. 

\begin{enumerate}
    \item[] \textbf{Conditioned on Y=1}
    For an individual belonging to group $a$ for whom is given that $Y=1$, the conditional expected DS utility is given by $E[V|\vec{x},Y=1]= d(\vec{x})v_{11}+(1-d(\vec{x}))v_{01}$. Furthermore, the distribution of individuals with $Y=1$ is given by $g(p)p/(\int g(p)p \: dp)$. Thus, the expected DS utility for a randomly selected individual with $Y=1$ is given by 
    \begin{equation}
        E[V|Y=1,a]=\frac{\int_0^1(d(p) \cdot p \cdot v_{11}+(1-d(p)) \cdot p \cdot v_{01})g(p) \:dp}{\int_0^1g(p) \:dp} = \frac{1}{BR_a}\int_0^1(d(p)(\tilde{\alpha}p+\tilde{\beta})+\tilde{\gamma} p+v_{00}) \: g(p) \: dp
    \end{equation}
    where $BR_a=P(Y=1|a)=\int_0^1g(p)p \:dp$, and $\tilde{\alpha}=v_{11}-v_{01}, \tilde\beta=0, \tilde\gamma=v_{01}$.
    
    Notice that this results in the same integral as for $J=\emptyset$, multiplied by a constant term ($1/BR_a$), and with slightly differently defined parameters $(\tilde{\alpha}, \tilde\beta, \tilde\gamma)$. Thus, the same arguments hold as for $J=\emptyset$ (see Appendix \ref{app: proof theorem threshold rule}), which proves Theorem \ref{thm: threshold decision rule} for $J=Y$.  
    
    \item[] \textbf{Conditioned on D=1} 
    For an individual belonging to group $a$ with $D=1$, the conditional expected DS utility is given by $E[V|\vec x, D=1,a]=p\cdot v_{11}+(1-p(\vec{x}))\cdot v_{10} =\tilde{\alpha}p(d)+\tilde{\beta}$, with $\tilde \alpha=v_{11}-v_{10}$ and $\tilde\beta=v_{10}$. Furthermore, the distribution of individuals with $D=1$ is given by $d(p) g(p)/(\int d(p) g(p) \: dp)$.    
    Following the reasoning of section \ref{subsec: utility-based}, the expected DS utility for a randomly selected individual with $D=1$ is given by 
    \begin{equation}\label{eq: int condition D=1}
    \begin{split}    E[V|D=1,a]&=\frac{\int_0^1(\tilde{\alpha}p+\tilde{\beta})d(p)g(p)dp}{\int_0^1d(p)g(p)dp}= \tilde{\alpha} \: \frac{\int_0^1pd(p)g(p)\: dp}{\int_0^1d(p)g(p) \: dp}+\tilde{\beta}\\
    \end{split}
    \end{equation}
    Again, we approximate this with the Riemann integral  
    \begin{equation}
        E[V|D=1,a]=\tilde{\alpha}\frac{\sum_{i=1}^Nd_ip_iw_i}{\sum_{i=1}^Nd_iw_i}+\tilde{\beta}
    \end{equation}
    with $w_i=\int_{\text{bin }i}g(p)dp$ the weight of each summand. 

    In the following, we derive the condition for which changes ($\Delta_j,\Delta_k)$ leave $E[V|D=1]$, or equivalently $\frac{\sum_{i=1}^Nd_ip_iw_i}{\sum_{i=1}^Nd_iw_i}$, unchanged:   
    \begin{equation}
        \tilde{\alpha}\frac{\Delta_j p_j w_j+\Delta_k p_k w_k+\sum_{i=1}^Nd_ip_iw_i}{\Delta_j w_j+\Delta_k w_k+\sum_{i=1}^Nd_iw_i} +\tilde{\beta}=\tilde{\alpha}\frac{\sum_{i=1}^Nd_ip_iw_i}{\sum_{i=1}^Nd_iw_i}+\tilde{\beta}
    \end{equation}
    With some algebra, it is easy to show that this condition is equivalent to 
    \begin{equation}\label{eq: Delta dj condition D=1}
        \Delta_j =\frac{w_k(\eta-p_k)}{w_j(p_j-\eta)}\Delta_k
    \end{equation}
    with $\eta=\frac{\sum_{i=1}^Nd_ip_iw_i}{\sum_{i=1}^Nd_iw_i}=PPV$.
        
    Assume $(\Delta_j,\Delta_k)$ fulfills this condition. Then, $E[U|a]$ changes by 
    \begin{equation}\label{eq: delta E(U) condition D=1}
    \begin{split}
    \Delta E[U|a] & = \Delta_j(\alpha p_j+\beta)w_j+\Delta_k(\alpha p_k+\beta)w_k \\
     & = \frac{w_k(\eta-p_k)}{w_j(p_j-\eta)}\Delta_k(\alpha p_j+\beta)w_j+\Delta_k(\alpha p_k+\beta)w_k\\
     & = \frac{(\eta\alpha+\beta)(p_j-p_k)}{p_j-\eta}w_k\Delta_k
    \end{split}
    \end{equation}

    Analogously to the proof of Theorem \ref{thm: threshold decision rule} for $J=\emptyset$ we assume an arbitrary decision rule $d\in \mathcal D_\eta$ with $E_d[V|D=1,a]=\eta$ and we will make 
    simultaneous changes $(\Delta_j,\Delta_k)$ to the decisions $(d_j,d_k)$ with $p_j<p_k$ that leave $E[V|D=1,a]$ unchanged, but improve $E[U|a]$.

    Note that $\eta\in[0,1]$, as $\eta=PPV$.  Let $p_j,p_k$ such that $p_j<\eta<p_k$. Then,  $\eta-p_k<0$ and $p_j-\eta<0$.  $E[U]>0$ is achieved for $\Delta_j,\Delta_k<0$ when $\eta<p^*$, and for $\Delta_j,\Delta_k>0$ when $\eta>p^*$ (Eq. \ref{eq: delta E(U) condition D=1}).

    \textbf{For $\eta<p^*$}, it is always possible to find a combination $(\Delta_j,\Delta_k)$ that lead to $E[U]>0$, unless $d==0$ for $d\in[0,\eta)$ or $d==0$ for $(\eta,1]$. In the first case ($d==0$ for $d\in[0,\eta)$), it is always possible to further increase $E[U]$ unless $d$ corresponds to a lower-bound threshold rule for  $ d\in(\eta,1]$ (see reasoning for the proof of $J=\emptyset$). In the second case ($d==0$ for $(\eta,1]$), $E[U]$ can be increased unless $d$ corresponds to an upper-bound threshold rule for $d\in[0,\eta)$. 

    \textbf{For $\eta>p^*$}, an analogous argumentation shows similarly that it is always possible to find a combination $(\Delta_j,\Delta_k)$ leading to $E(u)>0$, unless $d$ corresponds to an upper-bound or a lower-bound threshold rule. 
    
\end{enumerate}

\section{Population distribution - Section 4} \label{app: synthetic dataset sec 4}
In Section 4 we consider a population consisting of two groups 0 and 1. The population is equally distributed across the two groups, but both groups have a different distribution of the probability $P[Y=1]$. We assume to know the distributions of those probabilities for both groups: $g(p|a)$ follows the Beta distribution for both groups with parameters $\alpha_0=4.5,\beta_0=5.5$ for group 0 and $\alpha_1=5,\beta_1=3$ for group 1. Figure \ref{fig: pop distr sec 4} shows the distributions $g(p|a)$ of the groups.   

\begin{figure}[h!]
    \centering
    \includegraphics[width=0.4\linewidth]{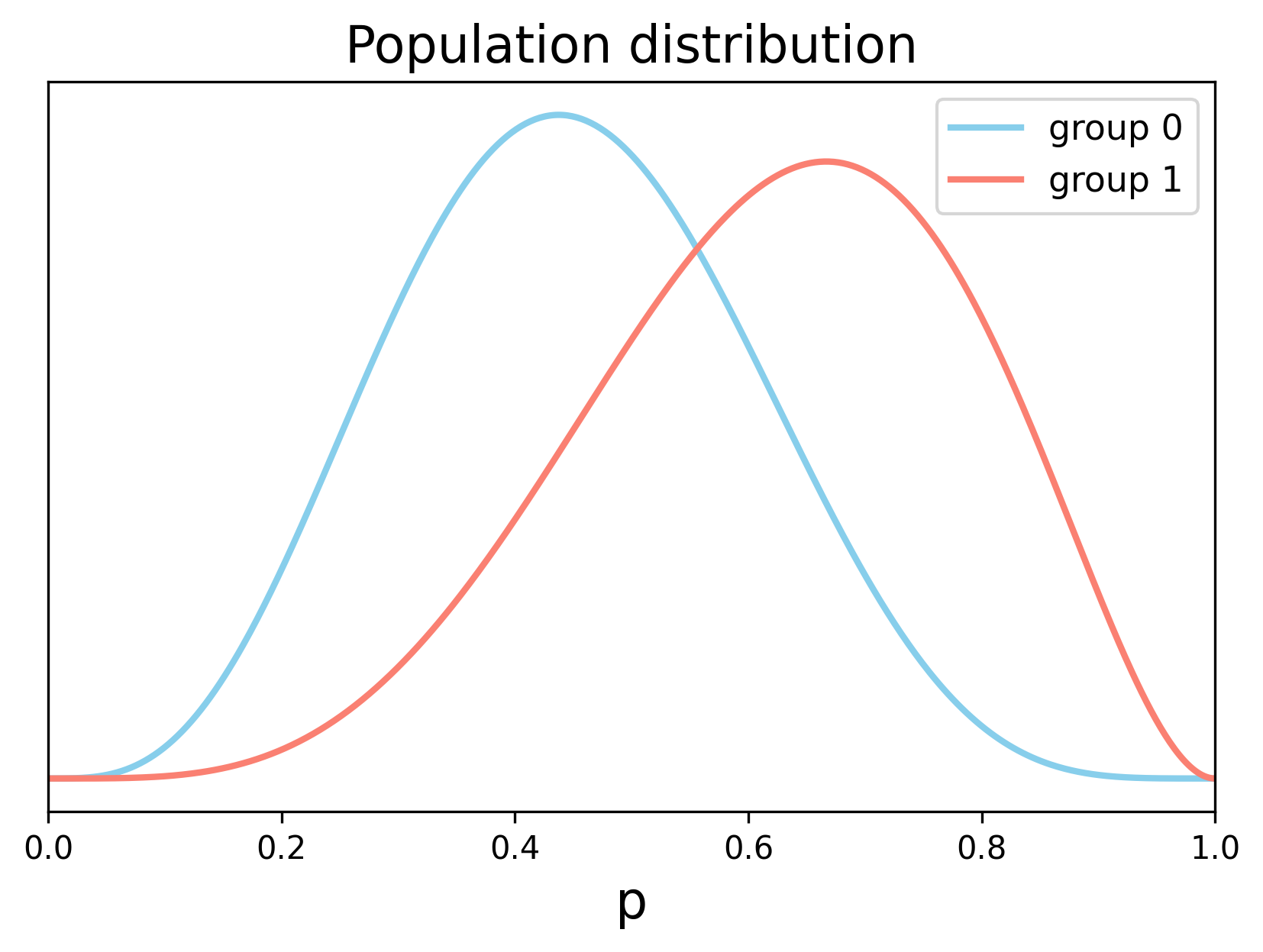}
    \caption{This plot shows the distributions $g(p|a)$ for the two groups $A=0$ and $A=1$.}
    \label{fig: pop distr sec 4}
\end{figure}

\end{document}